\documentclass{article}

\usepackage[main,final,nonatbib]{neurips_2026}

\usepackage{amsmath,amsfonts,bm}

\def\eqref#1{equation~\ref{#1}}

\def\1{\bm{1}}

\DeclareMathAlphabet{\mathsfit}{\encodingdefault}{\sfdefault}{m}{sl}
\SetMathAlphabet{\mathsfit}{bold}{\encodingdefault}{\sfdefault}{bx}{n}

\DeclareMathOperator*{\argmin}{arg\,min}

\usepackage{amsmath,amsfonts,bm}
\usepackage{dsfont}
\usepackage{caption}
\usepackage{subcaption}
\usepackage{verbatim}
\usepackage[dvipsnames]{xcolor}
\definecolor{orange}{rgb}{0.93725,0.52549,0.2117647}
\definecolor{blue}{rgb}{0.23137,0.4588,0.68627}
\usepackage{dirtytalk}
\usepackage{wrapfig}
\usepackage{color, colortbl}
\usepackage{stfloats}
\usepackage{graphicx}
\usepackage{multirow}
\usepackage{siunitx}
\usepackage{array}
\usepackage{amssymb}
\usepackage{longtable}
\usepackage{caption}
\usepackage{lscape}
\usepackage{adjustbox}
\newcommand{\thickhline}{%
    \noalign {\ifnum 0=`}\fi \hrule height 1pt
    \futurelet \reserved@a \@xhline
}
\newcolumntype{"}{@{\hskip\tabcolsep\vrule width 1pt\hskip\tabcolsep}}
\newcommand\norm[1]{\left\lVert#1\right\rVert}
\usepackage{upquote,textcomp,amsmath}

\usepackage[pagebackref=false,breaklinks=true,colorlinks,bookmarks=false]{hyperref}
\hypersetup{colorlinks,breaklinks,
            urlcolor=[rgb]{0.93725,0.52549,0.2117647},
            citecolor=[rgb]{0.23137,0.4588,0.68627},
            linkcolor=[rgb]{0.93725,0.52549,0.2117647}}
\usepackage[numbers]{natbib}
\usepackage[utf8]{inputenc}
\usepackage[T1]{fontenc}
\usepackage{hyperref}
\usepackage{url}
\usepackage{booktabs}
\usepackage{amsfonts}
\usepackage{nicefrac}
\usepackage{microtype}
\usepackage{xcolor}
\usepackage{amsthm}
\usepackage{physics}

\newtheorem{theorem}{Theorem}

\title{A Minimal Interpretable Architecture for Zero-Shot Reconstruction of Dynamical Systems}

\author{
    Christoph Jürgen Hemmer\textsuperscript{2,1,3}\footnotemark[1] ,
    Florian Plaswig\textsuperscript{1,3}\footnotemark[1] ,
    Daniel Durstewitz\textsuperscript{1,2,3,4}\\
    \textsuperscript{1}Dept. of Theoretical Neuroscience, Central Institute of Mental Health, Mannheim, Germany \\
    \textsuperscript{2}Interdisciplinary Center for Scientific Computing (IWR), Heidelberg, Germany \\
    \textsuperscript{3}Faculty of Physics and Astronomy, Heidelberg University, Heidelberg, Germany \\
    \textsuperscript{4}ELLIS Unit Heidelberg, Heidelberg, Germany \\  
}

\begin{document}
\renewcommand{\thefootnote}{\fnsymbol{footnote}}
\footnotetext[1]{These authors contributed equally to this work.\par Corresponding authors: \{christoph.hemmer, daniel.durstewitz\}@zi-mannheim.de \par \textbf{Code}: \url{https://github.com/DurstewitzLab/DynaBase}}
\renewcommand{\thefootnote}{\arabic{footnote}} 

\maketitle

\begin{abstract}
Recent foundation models (FMs) for zero-shot reconstruction of dynamical systems (DS) achieve strong out-of-domain generalization but provide little insight into the \textit{mechanisms} that underlie their forecasts. Such an understanding could help to strip down overladen FM architectures to their bare essence and expose the minimal requirements for in-context learning in the DS domain. Toward this goal, here we iteratively reduce a recent powerful SOTA model for DS reconstruction, DynaMix \cite{hemmer2025true}, to a minimal interpretable two-parameter form, which we call \emph{DynaBase}. DynaBase produces forecasts through a linear blend of the current latent state and the nearest in-context neighbor and its temporal successor. Surprisingly, despite its extreme simplicity, DynaBase produces highly competitive zero-shot DS reconstructions across chaotic, cyclic, and real-world systems, with a negligible parameter load, many orders of magnitude below that of other FMs. Even more, this extreme simplicity permits direct model optimization on DS reconstruction measures, as well as closed-form one-step analytical solutions on prediction MSE. Theoretical and empirical analysis of DynaBase further leads to a 1-parameter family of maps, with the context-parroting algorithm of \cite{zhang2026context} recovered at one end, and chaotic (divergent but bounded) behavior at the other. We further show how different training strategies lead to models either optimal for short-term prediction or for DS reconstruction. Thus, DynaBase not only exposes the minimal mechanisms required for producing zero-shot DS reconstruction, but also reconciles within an accessible mathematical frame divergent observations in the literature.
\end{abstract}

\section{Introduction}
Many, if not most, natural and engineered systems, from chemical and molecular processes to climate, ecosystems, brain activity, or stock markets, are naturally described as dynamical systems (DS) \cite{durstewitz_reconstructing_2023,gilpin_chaos_2022, ecology1, tziperman-97,BhallaIyengar1999Emergent,buzsaki_rhythms_2006,mandelbrot_misbehavior_2007}. Reconstructing these DS (DSR) in a data-driven way from just time series observations is of fundamental interest to all sciences, which seek tractable models that enable mechanistic insight into the underlying dynamical processes. DSR, hence, aims for \textit{generative surrogate models} of the underlying dynamics which not only forecast a few steps ahead, but---importantly---reproduce \textit{long-term statistical} properties of the DS in question, such as its attractor geometry (ergodic distribution) and power spectrum \cite{brenner_almost_2024,durstewitz_reconstructing_2023,gilpin_model_2023,goring_domain_2024,platt2023constraining,brunton_chaos_2017,brunton_data-driven_2019,durstewitz2026position}. While traditionally DSR involved custom-training ML/AI models on each particular dataset \cite{brenner_tractable_2022,hess_generalized_2023,goring_domain_2024,platt_systematic_2022,platt2023constraining}, inspired by the success of in-context learning in large language and time-series foundation models (FMs) \cite{ansari2024chronos,ansari2025chronos2univariateuniversalforecasting,bhethanabhotla2024mamba4cast}, a new line of research explored whether zero-shot generalization to previously unseen DS is also feasible \cite{hemmer2025true, lai2025panda}.

However, as for most if not all FMs, it is currently utterly unclear how SOTA DSR FMs like DynaMix \cite{hemmer2025true} achieve their zero-shot out-of-domain generalization, and what the essential architectural requirements for this are. Insight into how AI models solve their problems is of general interest, but is particularly important in scientific and medical domains where the focus is on mechanistic understanding \cite{Rudin2019,durstewitz_reconstructing_2023}. Here we tackle this question in the context of DSR by reverse-engineering the pretrained DynaMix model. We produce a series of systematic ablations and simplifications that largely retain, or even improve on, the model's zero-shot DSR performance. We ultimately arrive at an algebraically minimal form consisting just of a linear combination of the model's previous state, the nearest neighbor from the context signal, and its immediate temporal successor. A natural \emph{symmetry} requirement on the context points further condenses the model into a two-parameter form. We show that this recursive map, \emph{DynaBase}, can generate point, cyclic, and chaotic attractors, and allows for closed-form one-step least-squares fitting to time series data. Theoretical and empirical analysis of DynaBase trained directly for DSR reveals that optimal zero-shot solutions essentially lie along a 1-parameter curve with truly chaotic dynamics at one end and `context parroting' of a time series snippet at the other, a strategy recently observed in transformer-based time series FMs like Chronos \cite{zhang2026context,zhang2025zeroshotforecastingchaoticsystems}. Our observations are in line with a recent literature that demonstrates that tiny ML/AI models, many orders of magnitude below the parameter load of modern FMs, often achieve comparable, or even widely superior performance, on a range of tasks from cognitive and behavioral modeling \cite{li2025discovering} to challenging reasoning problems \cite{wang2025hierarchicalreasoningmodel,jolicoeurmartineau2025morerecursivereasoningtiny}.

In a nutshell, our main results are:
\begin{itemize}
    \item \textbf{Minimal zero-shot DSR model:} Through iterative simplification of DynaMix \cite{hemmer2025true} we create a minimal 2-parameter map, DynaBase, which exhibits competitive zero-shot DSR performance at essentially zero training costs.
    \item \textbf{Mechanistic insight:} The model's two coefficients $(\alpha,\beta)$ parametrize a continuous spectrum of dynamical behaviors, with optimal DSR solutions essentially driven to a 1-parameter subspace. This minimal form achieves its strong performance essentially by 1) closely following the context signal while 2) regulating local divergence rates through a single parameter $\alpha$, enabling a spectrum from point attraction and stable cycles to chaos.
    \item \textbf{Impact of training strategy:} We further demonstrate that training using either a standard ahead-prediction MSE loss vs. a specific DSR loss yields fundamentally different solutions: Training for ahead-prediction often collapses to simple context parroting and hence does not recover the true dynamics, especially for chaotic systems, which is only achieved by DSR-specific training.
\end{itemize}

\section{Related work}\label{sec:related}

\paragraph{Dynamical systems reconstruction (DSR)}
DSR aims to learn generative models from time series data that capture the underlying system's dynamical long-term behavior, i.e. its long-term or ``climate'' statistics and state space geometry \cite{platt2023constraining,durstewitz_reconstructing_2023,hess_generalized_2023,mikhaeil_difficulty_2022,gilpin_model_2023,brunton_chaos_2017}. Custom-trained models for this purpose build on library-based methods like SINDy \cite{brunton_discovering_2016,loiseau_constrained_2018}, reservoir computing \cite{pathak_using_2017, platt_systematic_2022, platt2023constraining,carroll_network_2019,han_tighter_2022}, neural ODEs \cite{chen_neural_2018, alvarez_dynode_2020, ko_homotopy-based_2023}, Koopman operators \cite{lusch_deep_2018, otto_linearly-recurrent_2019, azencot_forecasting_2020, naiman_koopman_2021, brunton_modern_2021, wang_koopman_2022}, or on various types of RNNs \cite{trischler_synthesis_2016,durstewitz_state_2017,vlachas_data-driven_2018,brenner_tractable_2022,rusch_long_2022,hess_generalized_2023,hemmer_optimal_2024,brenner_almost_2024,brandle2026continuous}. More important than the particular architecture appears to be the training strategy: sparse or generalized teacher forcing \cite{mikhaeil_difficulty_2022,hess_generalized_2023} are control-theoretic training techniques that enable long trajectory roll-outs even for chaotic DS, while specific regularization methods based on a system's invariant measures have been designed to enforce agreement in long-term statistics \cite{Fumagalli,platt2023constraining,jiang2023training}. All of these methods, however, require purpose-training on each new system, and out-of-domain generalization remains a key challenge \cite{goring_domain_2024}.

\paragraph{Zero-shot DSR and time-series foundation models (FMs)}
Multi-environment \cite{yin2021leads,kirchmeyer2022generalizing}, hyper-network-guided \cite{vermani2025metadynamical}, or hierarchical \cite{brenner2024learning} approaches have been used to train DSR models on many DS simultaneously, but all of these still require fine-tuning on new systems or parameter configurations not seen in training. More recently, inspired by the success of in-context inference in LLMs \cite{brown2020language,garg2022can,akyurek2023what}, a line of general-purpose time series (TS) FMs pretrained on large corpora of real and synthetic data has emerged, which forecast TS from a short observed context snippet zero-shot. Most of these, like Chronos \cite{ansari2024chronos,ansari2025chronos2univariateuniversalforecasting}, TimesFM \cite{das2024decoder}, or Panda \cite{lai2025panda}, are based on transformers, with a few exceptions that engage RNNs like the xLSTM \cite{beck2024xlstm} or Mamba \cite{gu2024mamba} as their core. All of these, however, lack the ability to reproduce a system's long-term behavior as required in DSR \cite{goring_domain_2024,hemmer2025true,durstewitz2026position}. They mostly either converge to just fixed points in the limit, or cyclically repeat the context \cite{zhang2025zeroshotforecastingchaoticsystems,zhang2026context}, and thus fail to capture DS with inherently chaotic dynamics \cite{hemmer2025true}. DynaMix \cite{hemmer2025true} closes this gap by pretraining, based on DSR-specific training techniques \cite{mikhaeil_difficulty_2022}, a mixture of AL-RNN \cite{brenner_almost_2024} experts, weighted by a context-dependent gating network. It delivers competitive short-term forecasts and SOTA long-term performance, while being relatively lightweight with $\sim 10^4$ parameters.

\paragraph{Mechanisms of in-context learning}
What is at the root of in-context learning has been intensely studied in the past few years \cite{brown2020language,garg2022can,akyurek2023what}, and hypotheses range from transformers implementing gradient-descent in-context \cite{Bai2023neurips,vonoswald2023transformers,lin2024dual,xie2021explanation} to mechanisms that more or less boil down to a simple remixing of their memory contents \cite{shen24d,wang2025can,wu25iclasassocmem}. For DSR FMs, the analogous problem remains largely open but---at the same time---appears more tractable due to the more constrained and mathematically better defined \cite{goring_domain_2024} nature of the question. Recent studies using Chronos suggested that TS FMs may long-term forecast DS by the simple mechanism of context parroting \cite{zhang2025zeroshotforecastingchaoticsystems}, a well known pattern in LLMs \cite{olsson2022context,edelman2024evolution,chen2024unveiling}. Based on this phenomenon, \citet{zhang2026context} design a `context parroting algorithm' which takes the last $D$ values of the context as a delay-embedded query, retrieves the best-matching motif elsewhere in the context, and copies the subsequent fragment as forecast. They show that this strategy works surprisingly well and actually outperforms most TS FMs. But context parroting, as exhibited by Chronos or conceptualized in this form, can only reproduce \textit{discrete-cyclic} activity and thus \textit{inherently} fails to capture defining signatures of chaotic systems, like positive Lyapunov exponents and rather broad power spectra \cite{hemmer2025true, durstewitz2026position}. A superficially similar principle is employed in classical nonlinear forecasting algorithms that retrieve from a library (history) locally neighboring trajectories whose next time step predictions are either just averaged or based on which a locally linear model is fit \cite{farmer_predicting_1987,kantz_nonlinear_2004}. Crucially, in contrast to mere context parroting, this neighborhood may change, however, with each prediction step taken. Through iterative reduction of DynaMix we arrive at a mathematical form that contains these classes of algorithms as special cases.

\section{Methods}\label{sec:methods}

\subsection{Background: DynaMix}
DynaMix \cite{hemmer2025true} is a DSR foundation model based on a mixture-of-experts architecture, using $j=1 \dots J$ almost-linear RNNs (AL-RNNs) \cite{brenner_almost_2024} as experts, 
\begin{equation} \label{eq:alrnn}
    \bm{z}^j_{t+1}= F^\text{exp}_j(\bm{z}_t)=\bm{A}_j \bm{z}_t + \bm{W}_j \Phi^*(\bm{z}_t)+\bm{h}_j\;,
\end{equation}
where $\bm{z}^j_t\in\mathbb{R}^M$ is an $M$-dimensional latent state, $\Phi^*(\bm{z}_{t}) := \left[ \bm{z}_{1:M-P,t},\ \mathrm{ReLU}(\bm{z}_{M-P+1:M,t})\right]^\top$ is the activation function with $M-P$ linear units and $P$ ReLUs, and $\bm{A}_j\in \mathrm{diag}(\mathbb{R}^M)$, $\bm{W}_j\in\mathbb{R}^{M\times M}$, and $\bm{h}_j\in\mathbb{R}^{M}$ are learnable parameters. The different experts are combined into an overall next state prediction by $\bm{z}_{t+1}=\sum_{j=1}^J w^{exp}_{j,t}\cdot\bm{z}_{t+1}^j$, where the weights $w^{exp}_{j,t} \in (0,1)$ are determined by a gating network through 
\begin{equation}\label{eq:exp_weights}
    \bm{w}^{exp}_t(\bm{z}_t,\tilde{\bm{C}}\bm{w}^{att}_t)=\sigma\left(\frac{\text{MLP}(\tilde{\bm{C}}{\bm{w}^{att}_t},\bm{z}_t)}{\tau_{\text{exp}}}\right)\in\mathbb{R}^J\;,
\end{equation}
where $\tau_{\text{exp}}$ is a temperature parameter and $\tilde{\bm{C}}=\text{CNN}(\bm{C})$ are temporal features obtained from the provided context signal $\bm{C}=\{\bm{c}_t\}\in\mathbb{R}^{N\times T_C}$ through a CNN. Attention weights $\bm{w}^{att}_t$ at each time step $t$ are in turn computed based on a distance between current latent state $\bm{z}_{t}$ and the context $\bm{C}$ as
\begin{equation}\label{eq:attn_weights}
    \bm{w}^{att}_t(\bm{z}_t,\bm{C}) = \sigma\left( 
        \frac{\left| \bm{C} - \left( \bm{Dz}_t + \bm{\epsilon} \right)\bm{1}_{T_C}^\top \right|^\top\bm{1}_N}{\tau_{\text{att}}}
    \right) \in \mathbb{R}^{T_C}\;,
\end{equation}
where $\tau_{\text{att}}$ is another temperature, $\bm{D}\in\mathbb{R}^{N\times M}$ a learnable matrix, $\bm{\epsilon}\sim\mathcal{N}(0,\bm\Sigma)$ exploration noise, and $\bm{1}_{\{T_C,N\}}$ are column vectors of ones of length $T_C, N$, respectively. In \cite{hemmer2025true}, DynaMix was trained on $\approx 6 \times 10^5$ trajectories from $34$ $3d$ DS in the cyclic or chaotic regime using control-theoretic training techniques that ensure DSR criteria are met \cite{mikhaeil_difficulty_2022,hess_generalized_2023}. After pretraining, given any context time series snippet $\bm{C}$, DynaMix is charged with forecasting the long-term trajectory evolution of the underlying DS without any retraining or fine-tuning.

\subsection{Reducing DynaMix}\label{sec:ablations}
At a high level, DynaMix can be compactly expressed as a single RNN of the following form 
\begin{equation}
    \bm{z}_{t+1} = F_\text{DynaMix}(\bm{z}_t, \bm{C}) := F^\text{exp}(\bm{z}_t)\cdot \bm{w}^{exp}_{t}(\bm{z}_t,\tilde{\bm{C}}\bm{w}^{att}_t)
    \label{eq:update-rule-dynamix}
\end{equation}
with (see \cite{hemmer2025true})
\begin{equation}
    \tilde{\bm{C}}{\bm{w}^{att}_t} = \text{CNN}(\bm{C})\cdot \bm{w}_t^\text{att}(\bm{z}_t,\bm{C})\;.
    \label{eq:attention-weights-dynamix}    
\end{equation}
Note that eq. \ref{eq:update-rule-dynamix} is a weighted sum of \textit{piecewise-linear} systems $F_j^\text{exp}$, where, however, the weights $\bm{w}^{exp}_t(\bm{z}_t,\tilde{\bm{C}}\bm{w}^{att}_t)$ themselves are nonlinear functions of the current state and the context. To reduce this model, we first replace the \textit{almost-linear} experts $F^{exp}_j$ (eq. \ref{eq:alrnn}) by a \textit{fully linear} form, without loss in performance as confirmed by ablations (Table \ref{tab:dynamix_ablations}). Regarding the expert weights $\bm{w}^{exp}$, it has previously been demonstrated \cite{hemmer2025true} that the MLP in eq.~\ref{eq:exp_weights} can also be replaced by a linear layer without significant loss in performance. Likewise, we observe that the output nonlinearity $\sigma$ (softmax; eq.~\ref{eq:exp_weights}) is nearly linear over its relevant operating domain and hence replace it by a linear function as well (Table \ref{tab:dynamix_ablations}). This leaves us with a representation that is bilinear in $\bm{z}_t$. 

We assume that this bilinear form on a compact domain can be mimicked by a universal ReLU-based MLP approximator \cite{yarotsky2017error}. Since the MLP is ReLU-based, for each given input it performs a linear mapping, and hence its `degree of non-linearity' depends on how many of these linear subregions (polytopes) are actively engaged. We observe that assuming just a single linear region, i.e. a fully linear MLP, does not degrade performance either (Table \ref{tab:dynamix_ablations}). Thus, $\tilde{\bm{C}}=\text{CNN}(\bm{C})$ and $\bm{w}^{att}_t(\bm{z}_t, \bm{C})$ remain as the only nonlinearities, and the whole model can be expressed as:
\begin{equation}\label{eq:dynamix_linear}
    \bm{z}_{t+1}=F_\text{Linear}(\bm{z}_t,\bm{C})=\bm{A}\bm{z}_t+\bm{B}(\text{CNN}(\bm{C}){\bm{w}^{att}_t}(\bm{z}_t, \bm{C}))+ \bm{h}\;.
\end{equation}

Next, for $\tau_\text{att} \to 0$, $\bm{w}^{att}_t(\bm{z}_t, \bm{C})$ essentially boils down to a nearest neighbor selector. Indeed, empirically $\tau_\text{att} \approx 0.06\ll1$, and replacing it by $\tau_\text{att}\to0$ does not degrade DynaMix performance (Table \ref{tab:dynamix_ablations}). Further assuming a CNN kernel size of $2$ and noticing that for a single linear CNN layer (as used in DynaMix) the CNN operations can be combined with $\bm{B}$ in eq. \ref{eq:dynamix_linear}, we thus obtain 
\begin{align}
    f_\text{Linear}(\bm{z}_t,\bm{C})&=\bm{A}\bm{z}_t+\bm{B}^{(1)}\bm{c}_{\tau(\bm{z}_t)}+\bm{B}^{(2)} \bm{c}_{\tau(\bm{z}_t)+1}+\tilde{\bm{h}} \label{eq:linear-general} \\
    \tau(\bm{z}_t)&=\argmin_{s<T_C}\norm{\bm{c}_s-\bm{z}_t}_2 \label{eq:tau}
\end{align}
where $\bm{A},\bm{B}^{(1)},\bm{B}^{(2)} \in \mathbb{R}^{M\times M}$ and $\tilde{\bm{h}}\in\mathbb{R}^M$. We call this form a \textit{recursive affine nearest-neighbor (NN) map}. See Appx. \ref{appx:ablation} for more details on the reduction.

\subsection{DynaBase}\label{sec:dynabase}
In DSR, two vector fields that differ only by a rigid coordinate transformation and possibly invertible coordinate changes by some matrix $\bm{T}$ describe the same underlying flow, yielding \textit{topologically conjugate} reconstructions \cite{perko_differential_2001}. Recognizing this invariance, we can state the following result:
\begin{theorem}[Uniqueness of recursive affine NN maps]\label{thm:uniqueness}
Any nearest-neighbor recursive map of the form eq. \ref{eq:linear-general} that is affine in $\bm{z}_t$, $\bm{c}_{1}$, and $\bm{c}_{2}$, equivariant under $O(N)$ rotations and reflections, and translation equivariant, belongs to the two-parameter linear family $f_{\alpha,\beta}$,
\begin{equation}
    \bm{z}_{t+1} = f_{\alpha,\beta}(\bm{z}_t) = \alpha\, \bm{z}_t + \beta\, \bm{c}_{1} + (1-\alpha-\beta)\, \bm{c}_{2}.
\end{equation}
\end{theorem}
\begin{proof}See Appx. \ref{sec:proofs}.\end{proof}
Exploiting Thm. \ref{thm:uniqueness}, we obtain a minimal model, which we call \textit{DynaBase}, as 
\begin{equation}
    \bm{z}_{t+1} = f_{\alpha,\beta}(\bm{z}_t) = \alpha\, \bm{z}_t + \beta\, \bm{c}_{\tau(\bm{z}_t)} + \gamma\, \bm{c}_{\tau(\bm{z}_t)+1},\qquad \tau(\bm{z}_t)=\arg\min_{s<T_C}\norm{\bm{c}_s-\bm{z}_t}_2, \label{eq:dynabase}
\end{equation}
with $\alpha,\beta,\gamma\in\mathbb R$ and $\alpha+\beta+\gamma=1$. Thus, $\gamma=1-\alpha-\beta$ leaves $(\alpha,\beta)\in\mathbb{R}^2$ as the only free parameters of the minimal model (formalized in Thm. \ref{thm:uniqueness}). If the data were noise-free, a natural \emph{self-consistency} condition would be given by 
\begin{equation}
    \bm{c}_{i+1} = f_{\alpha,\beta}(\bm{c}_i)\;\forall i \in [1,T_C-1] \ ,
\end{equation}
which guarantees that whenever $\bm{z}_t$ coincides with a context point $\bm{c}_i$, update eq. \ref{eq:dynabase} maps $\bm{z}_t\mapsto \bm{c}_{i+1}$ exactly. Using this reasonable assumption, the affine NN map could be simplified further by imposing $\alpha = -\beta \implies \gamma=1,$ such that the whole construction depends only on a single parameter:
\begin{equation}
    \bm{z}_{t+1}=f_\alpha(\bm{z}_t)= \alpha\,(\bm{z}_t - \bm{c}_{\tau(\bm{z}_t)})+\bm{c}_{\tau(\bm{z}_t)+1} \ , \label{eq:self-consistent-one-param}
\end{equation}
where the local rate of con- or divergence is determined solely by $\alpha$. 

\subsection{Training}\label{sec:training}
For training DynaBase, we assume we are given $N$-dimensional training data $\bm{X}\in\mathbb{R}^{N\times T}$, of which we define the first $T_C$ time steps as the context signal $\bm{C}\in\mathbb{R}^{N\times T_C}$. Given the model's simplicity, it allows for straightforward training, where we assume $\bm{z}_t=\bm{x}_t$ (i.e., the dynamics is directly defined in the observation space). Specifically, we explore the two following approaches:

\paragraph{Linear regression} 
As DynaBase is linear in parameters, its one-step-ahead prediction error loss is convex and we can optimize it by least-squares regression in one step using the reparameterized form 
\begin{equation}
    \bm{z}_{t+1} - \bm{c}_{\tau(\bm{z}_t)+1} = \alpha\,(\bm{z}_t - \bm{c}_{\tau(\bm{z}_t)+1}) + \beta\,(\bm{c}_{\tau(\bm{z}_t)}-\bm{c}_{\tau(\bm{z}_t)+1}). \label{eq:linear-regression}
\end{equation}
Given a training sequence, nearest neighbors $(\bm{c}_{\tau(\bm{x}_t)},\bm{c}_{\tau(\bm{x}_t)+1})\;\forall\bm{x}_i\in\bm{X}$ are simply precomputed based on eq. \ref{eq:tau}, and ordinary-least-squares is used to solve for $(\alpha,\beta)$ in closed form. No backpropagation or teacher forcing is required; fitting is embarrassingly cheap compared to conventional FM pretraining.

\paragraph{Grid search}
However, linear regression only optimizes for short term forecasts, and hence might not be optimal for DSR problems which require capturing long-term behavior. We therefore investigate a second strategy, where the loss is given through a long-term reconstruction measure \cite{koppe_identifying_2019,mikhaeil_difficulty_2022}
\begin{equation} \label{eq:KL_loss}
    \text{Loss}(\bm{X},\bm{\hat{X}})= D_{\mathrm{stsp}}(p_{\mathrm{true}}(\bm{X})\|p_{\mathrm{pred}}(\bm{\hat{X}})) \ ,\;
\end{equation}
i.e. the Kullback-Leibler divergence between the true and model-generated trajectory distributions based on long-term rollouts $\bm{\hat{X}}$ of DynaBase. $D_{\mathrm{stsp}}$ is a common measure in the DSR literature to assess the agreement in true and reconstructed attractor geometries \cite{koppe_identifying_2019, mikhaeil_difficulty_2022, gilpin_model_2023,hess_generalized_2023,pals2024inferring, zhang2025zeroshotforecastingchaoticsystems,lai2025panda}. Since the model has only 2 trainable parameters, searching for an optimal solution by grid search is fairly straightforward and inexpensive. In Appx. \ref{sec:alternative_optimization} we furthermore show that using the sliced Wasserstein distance instead as optimization criterion, yet evaluating performance on $D_{\mathrm{stsp}}$, gives similar results.

\section{Results}\label{sec:results}

\subsection{Theoretical results}\label{sec:results-theory}
We now provide a theoretical characterization of the forecasting properties of DynaBase.
\paragraph{Structural characterization}
By Thm. \ref{thm:uniqueness}, DynaBase is a unique nearest-neighbor affine family compatible with the symmetry assumptions in Sec. \ref{sec:dynabase}. We use this structural result as the starting point for the two main mechanism-level consequences below.
\begin{theorem}[Reduction to context parroting]\label{thm:parroting}
Let $\bm{C}_{1:T_C}\in \mathbb{R}^{1\times T_C}$ be a context sequence and $\overline{\bm{C}}_{1:T_C-D+1}\in \mathbb{R}^{D\times (T_C-D+1)}$ its delay embedding with dimension $D$, defined as $\overline{\bm{c}}_s=(c_s,c_{s+1},\dots,c_{s+D-1})$. Let $f_0$ denote DynaBase with $\alpha=0$ (hence $\beta=0,\gamma=1$) on the embedded context $\overline{\bm{C}}$, and consider the discrete-time system $\bm{z}_{t+1}=f_0(\bm{z}_t)$ with initial condition $\bm{z}_0=\overline{\bm{c}}_{T_C-D+1}$ and $\tau(\bm{z}_0)\in[1,T_C-2D]$. Then the projection $x_t=z_{D,t}$ exactly reproduces the context-parroting algorithm of~\cite{zhang2026context}.
\end{theorem}
\begin{proof}See Appx. \ref{sec:proofs}.\end{proof}
Theorem \ref{thm:parroting} shows that the context-parroting baseline of \cite{zhang2026context} is recovered for free from DynaBase by simply setting $(\alpha,\beta)=(0,0)$ and delay-embedding the context. Any DynaBase model with $\alpha\neq 0$ is therefore strictly more expressive than parroting in the sense that a larger class of DS can be represented.

Given that the map $f_\alpha$ is piecewise-affine, we can also derive an exact continuous-time analogue \cite{monfared_transformation_2020} $\bm{\zeta}(t)$, such that $\bm{\zeta}(kh)=\bm{z}_k \forall k \in \mathbb{Z}, h \in \mathbb{R}$, by integrating the equations 
\begin{equation}
    \dv{\bm{\zeta}}{t} = \frac{\log\alpha}{h}\bm{\zeta}-\frac{1}{h}\frac{\log{\alpha}}{1-\alpha}(\bm{c}_{\tau(\zeta)+1}-\alpha \bm{c}_{\tau(\zeta)})\label{eq:continuous-limit}
\end{equation}
where we denote the discrete timestep size of $f_\alpha$ as $h$ and $\frac{\log\alpha}{1-\alpha}\rvert_{\alpha=1}$ is defined by its continuous extension. Notably, this limit \emph{does not exist} for $\alpha = 0$ (context parroting), and is only defined in the complex number space for $\alpha < 0$. For $\alpha = 0$, the map collapses onto $\bm{C}$ at every step, i.e. hops between a \textit{discrete set of points}, which is not reproducible by a continuous flow unless $\bm{C}$ contains exactly one (fixed) point. Importantly, context parroting is therefore an \textit{inherently discrete} special case of DynaBase, \textit{defined on a measure-zero set}, which cannot reproduce true limit cycles or chaotic behavior. A detailed derivation is given in Appx. \ref{sec:continuous_derivation}.

\paragraph{Dynamical characterization}
We now characterize the dynamical properties of the self-consistent DynaBase -- specifically how the parameter $\alpha$ governs the transition from trivial to genuinely chaotic behavior.
\begin{theorem}[Long-term dynamics]\label{thm:chaos}
For $\alpha=0$, every orbit of DynaBase converges after finitely many steps onto a terminal cycle contained in the context sequence, and thus has max. Lyapunov exponent $\lambda_\text{max}<0$.
\end{theorem}
\begin{proof}See Appx. \ref{sec:proofs}.\end{proof}
The case $\alpha=0$ is therefore dynamically trivial: all orbits eventually collapse to $\bm{C}$; context parroting is thus \emph{structurally incapable of producing chaos}.

\begin{theorem}[Boundedness]\label{thm:boundedness}
Let $f_\alpha$ be the self-consistent DynaBase map with parameter $\alpha \geq 0$ and context sequence $\bm{C}_{1:T_C}\subset\mathbb{R}^N$, and let $d_C(z) := \min_{1\leq s\leq T_C-1}\|z-c_s\|_2$. Then $\alpha \in [0,1)$ is a sufficient but not a necessary condition for the orbits of $f_\alpha$ to remain bounded, while for $\alpha = 1$ this holds if the context is closed, i.e., $d_{\bm{C}}(\bm{c}_{T_C}) = 0$. More precisely, $\alpha \in [0,1]$ guarantees $d_{\bm{C}}(\bm{z}_{t+1}) \leq \alpha\, d_{\bm{C}}(\bm{z}_t) + d_{\bm{C}}(\bm{c}_{T_C})$ for all $t$, where the last term vanishes for closed contexts. For $\alpha > 1$, there exist Voronoi parcellations by $\bm{C}$ which allow for strictly chaotic dynamics, i.e. such that $\lambda_{\max} > 0$ and the dynamics remains bounded.
\end{theorem}
\begin{proof}See Appx. \ref{sec:proofs}.\end{proof}
The theorem thus establishes that $\alpha\in[0,1)$ is sufficient for the orbit to remain confined to -- or track -- the context sequence, as distances to $\bm{C}$ remain bounded. For $\alpha = 1$, distances to $\bm{C}$ are non-increasing whenever the context is closed. For $\alpha>1$, nearby trajectories diverge almost everywhere. If the arrangement of Voronoi cells through the context $\bm{C}$ is, however, such that its dynamics still remains bounded, it will almost surely follow a chaotic attractor \cite{alligood_chaos_1996,guckenheimer_nonlinear_1983}. An example is discussed in the proof to Thm. \ref{thm:boundedness} and illustrated in Fig. \ref{fig:alpha-histogram}c.

\subsection{Empirical forecasting comparison}\label{sec:results-empirical}
\begin{figure*}[!ht]
    \centering
    \includegraphics[width=0.99\linewidth]{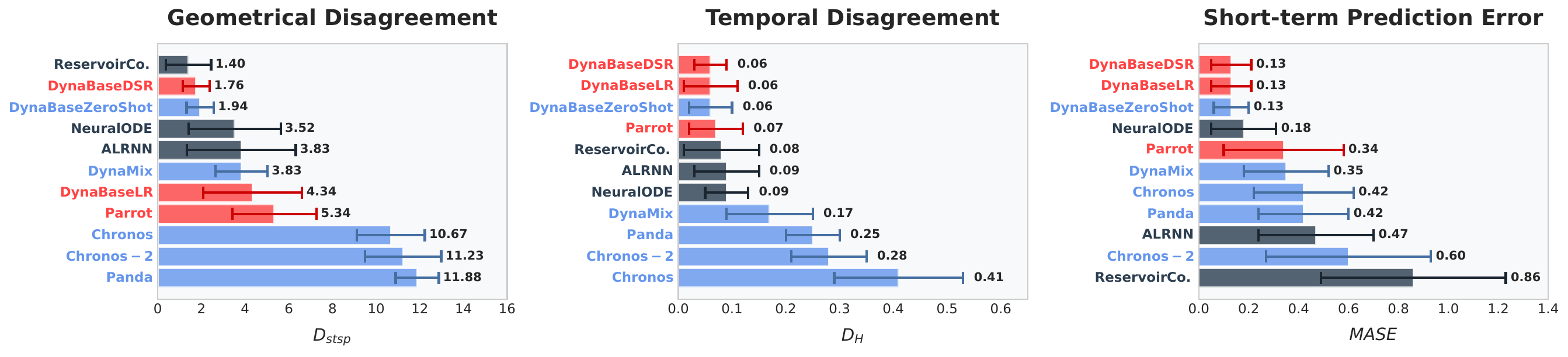}
    \caption{DSR performance across all $54$ test DS used in \cite{hemmer2025true}. Minimal models (red: DynaBase training variants and pure context parroting), custom-trained DSR models (gray), and pretrained zero-shot foundation models including a zero-shot variant of DynaBase (blue) are compared for the same context length $T_C=2000$ (values for custom-trained and foundation models taken from \cite{hemmer2025true}). Median$\pm$MAD of $D_{\mathrm{stsp}}$ (left, geometrical disagreement), $D_H$ (center, temporal disagreement) and MASE (right, $10$-step prediction error).}
    \label{fig:performance}
\end{figure*}
We evaluate DynaBase on the same DSR test set as in \cite{hemmer2025true}, reporting the same three complementary measures, two of them specific to DSR (see Appx. \ref{appx:metrics} for definitions): the disagreement $D_{\mathrm{stsp}}$ in state space geometry, the long-term temporal disagreement $D_H$, and the short-term mean absolute scaled error (MASE). To evaluate the long-term statistics $D_{\mathrm{stsp}}$ and $D_H$, each system is forecast for $T=10{,}000$ time steps, following \cite{hemmer2025true} and in agreement more generally with the DSR literature \cite{hess_generalized_2023,platt2023constraining,durstewitz2026position} (note that point forecasts become meaningless for chaotic systems beyond a few Lyapunov times). As introduced in Sec. \ref{sec:training}, we fit DynaBase by either (i) closed-form one-step linear regression, optimizing the one-step MSE in eq. \ref{eq:linear-regression}, or (ii) directly for DSR over the grid of parameter values. As baselines we include context parroting as reported in \cite{zhang2025zeroshotforecastingchaoticsystems,zhang2026context}, DynaMix \cite{hemmer2025true}, Panda \cite{lai2025panda} and Chronos \cite{ansari2024chronos, ansari2025chronos2univariateuniversalforecasting} as FMs, and three custom-trained DSR baselines (AL-RNN \cite{brenner_almost_2024}, reservoir computers \cite{patel_using_2023}, and Neural ODEs \cite{chen_neural_2018}) trained directly on the context signal provided to the FMs. In Appx. \ref{sec:classical_models} we furthermore report comparisons to `classical' library-based nonlinear forecasting and local approximation methods \cite{kantz_nonlinear_2004,farmer_predicting_1987}.

As shown in Fig. \ref{fig:performance}, DynaBase generally performs best, with the variant optimized for DSR (ii) better than the one trained for ahead-prediction (i) on the long-term measure $D_\text{stsp}$ (while the differences in MASE and $D_H$ are not statistically significant, Wilcoxon signed-rank test, $p>0.1$). However, for longer (but not too long) prediction horizons, optimizing for ahead-prediction achieves lower MASE than DSR-optimization (Table \ref{tab:dynabse_short-term}). This difference illustrates that, not unexpectedly, the choice of \emph{training objective} determines the model's forecasting strategy: a one-step MSE penalizes any local deviations from the ground truth, and hence tends to produce smaller $\alpha$ values enabling the map to track the immediate next step tightly, yielding better short-term predictions. In contrast, the grid optimization strategy (ii) directly targets the invariant attractor geometry, outperforming (i) on DSR (see also Appx. \ref{sec:alternative_optimization}, where we further show this also holds when optimizing for the Wasserstein distance yet evaluating on $D_{\mathrm{stsp}}$). In either case, training DynaBase with its few parameters is extremely cheap and rapid. 

\begin{figure*}[!ht]
    \centering
    \includegraphics[width=0.99\linewidth]{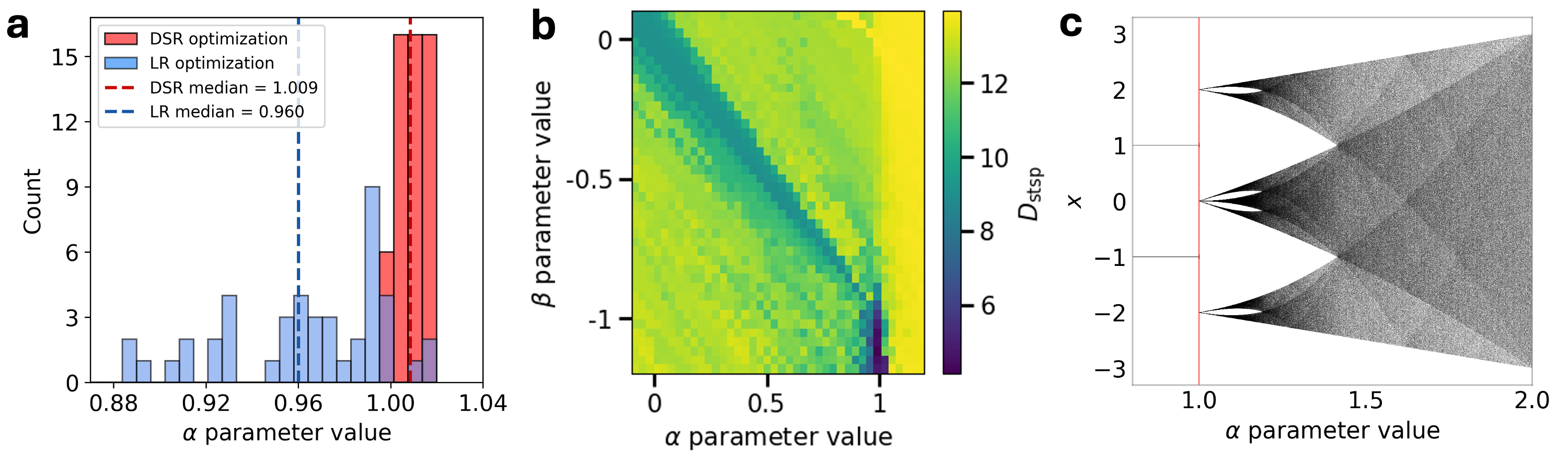}
    \caption{\textbf{a}) Comparison of distributions of estimated $\alpha$ values across all $54$ chaotic test DS \cite{hemmer2025true} for the two training objectives (sect. \ref{sec:training}). Closed-form minimization of the 1-step MSE yields $\alpha<1$, outside of the chaotic regime. In contrast, optimization for DSR places most $\alpha$ values firmly (and correctly) into the chaotic regime, with a median $\alpha\approx1.01$. \textbf{b}) DSR loss landscape (optimization for $D_\text{stsp}$) across $(\alpha,\beta)$ shows a valley at $\alpha=-\beta$ in agreement with the self-consistency condition, as well as a minimum at $\alpha>1$. \textbf{c}) Bifurcation diagram of the self-consistent DynaBase map $f_\alpha$ for the $1d$ context sequence $\bm{C}=(-1,1,-1)$ (the boundedness of this map was shown in the proof of Thm. \ref{thm:boundedness}). The graph shows the map's long-term behavior (attractor structure) as a function of $\alpha$, obtained by drawing $1{,}000$ random initial conditions from the interval $[-2,2]$ and discarding transients. The red line ($\alpha=1$) marks the transition from a stable $2$-cycle $(-1,+1)$ to the chaotic regime.}
    \label{fig:alpha-histogram}
\end{figure*}

The distribution of estimated $\alpha$ values across the test systems in Fig. \ref{fig:alpha-histogram}\textbf{a} makes this point explicit. The two training objectives lead to qualitatively different parameter regimes: linear regression on the 1-step MSE places the map mostly into the parroting/ semi-cyclic regime with $\alpha<1$, while optimizing directly on $D_{\mathrm{stsp}}$, in contrast, produces mostly chaotic behavior (in accordance with the ground-truth DS) with $\alpha>1$. This implies that only DSR-based training has correctly inferred the chaotic ground truth systems, as further confirmed by a high correlation ($r \approx0.90$) between max. Lyapunov exponents of the GT DS and the DynaBase map when fitted by DSR, but not when fitted by least-squares ($r \approx0.18$), Fig. \ref{fig:lyapunov_histogram}\textbf{a}. Moreover, the loss landscape reveals an elongated valley along the $\alpha=-\beta$ axis (Fig. \ref{fig:alpha-histogram}\textbf{b}), providing empirical support for the self-consistency condition in eq. \ref{eq:self-consistent-one-param}, which reduces DynaBase to a one-parameter map. The bifurcation graph in Fig. \ref{fig:alpha-histogram}\textbf{c} further illustrates on a synthetic toy example the cyclic behavior for $\alpha<1$ and emergence of chaos for $\alpha>1$. 

\begin{figure*}[!ht]
    \centering
    \includegraphics[width=0.99\linewidth]{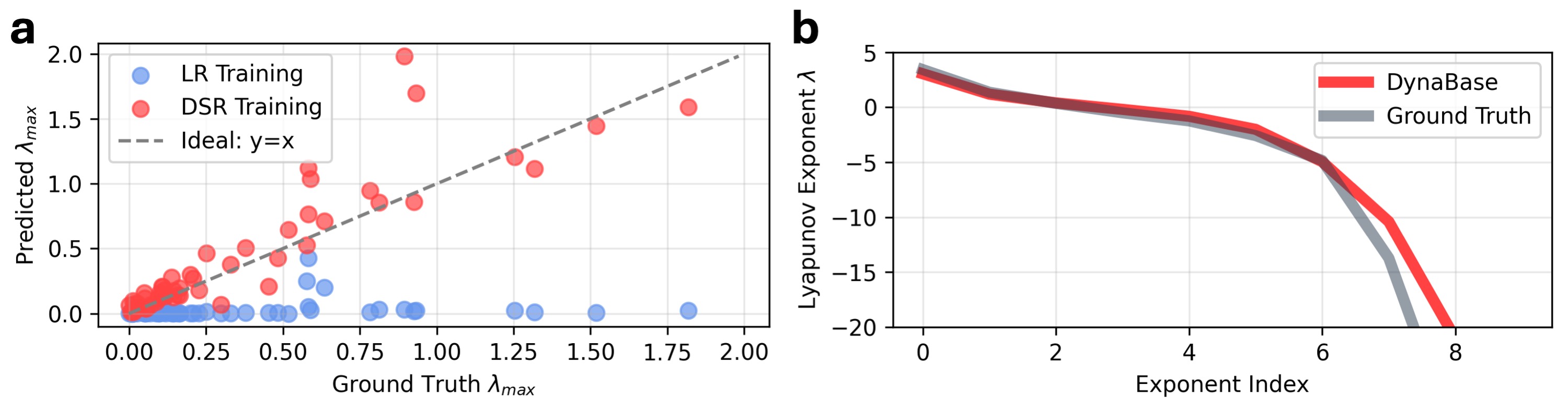}
    \caption{\textbf{a}) Lyapunov exponents estimated via the Rosenstein algorithm (see Appx. \ref{sec:lyapunov}) across test DS for both linear regression (blue) and DSR-based training (red) vs. ground truth exponents. Only DSR training produces long-term dynamics that closely follows that of the ground truth, including chaotic attractors with $\lambda_{max}>0$. \textbf{b}) Full Lyapunov spectrum of multiscale Lorenz-96 for DynaBase and ground truth, estimated via the Sano-Sawada method.}
    \label{fig:lyapunov_histogram}
\end{figure*}

Another interesting observation is that for DSR-specific training, $\alpha$ tightly concentrates on a single value near $\alpha\approx1.01$, across the diverse set of DS used. This suggests that, at least on this benchmark, the $\alpha$-parameter may be \emph{fixed a priori}. We test this by estimating $\alpha$ on a \textit{separate training set} (similar to the one used for training DynaMix in \cite{hemmer2025true}), obtaining $\alpha=1.006$ close to the median in Fig. \ref{fig:alpha-histogram}\textbf{a}, which leads to a \textit{parameter-free zero-shot forecaster} that requires no DS-specific fitting. Applying this to the context signals from the test set, we still achieve performance comparable to the custom-trained models (see Fig. \ref{fig:performance}, and Figs. \ref{fig:reconstruction1}-\ref{fig:reconstruction4} for zero-shot reconstructions). In Tables \ref{tab:scaling}-\ref{tab:noise} we, moreover, confirm that the estimated $\alpha$ value as well as DynaBase' performance is rather robust w.r.t. different scalings of the dimensions, context lengths, partial observations, and various noise levels.

Given that DynaBase only uses a single parameter $\alpha$ to regulate local divergence rates, how well could it represent the full Lyapunov spectra of high-dimensional systems? We tested this by applying the zero-shot DynaBase version to the $10$-dimensional, hyperchaotic two-scale Lorenz-96 system, with two positive Lyapunov exponents and slow and fast variables. As shown in Fig. \ref{fig:multiscale_lorenz} and Table \ref{tab:lorenz96_multiscale}, DynaBase reconstructs the hyperchaotic attractor and captures both time scales equally well ($D_{H}=0.13$ for slow and $D_{H}=0.11$ for fast variables). Even more surprisingly, the reconstructed Lyapunov spectrum closely matches that of the ground truth across its whole range (Fig.~\ref{fig:lyapunov_histogram}\textbf{b}).\footnote{Note that con-/divergence rates also depend on the piecewise Voronoi parcellation \cite{Coombes2024OscillatoryNetworksPWL}, not just on parameter $\alpha$.} 

Finally, we assessed DynaBase on long-term forecasting of \textit{real-world TS}, following the setup in \cite{durstewitz2026position}. The results in Fig. \ref{fig:TS_example} and Table \ref{tab:TS_eval} indicate that DynaBase often even outperforms DynaMix on these partially observed, noisy time series, preserving their long-term behavior.

\begin{figure*}[!ht]
    \centering
    \includegraphics[width=0.95\linewidth]{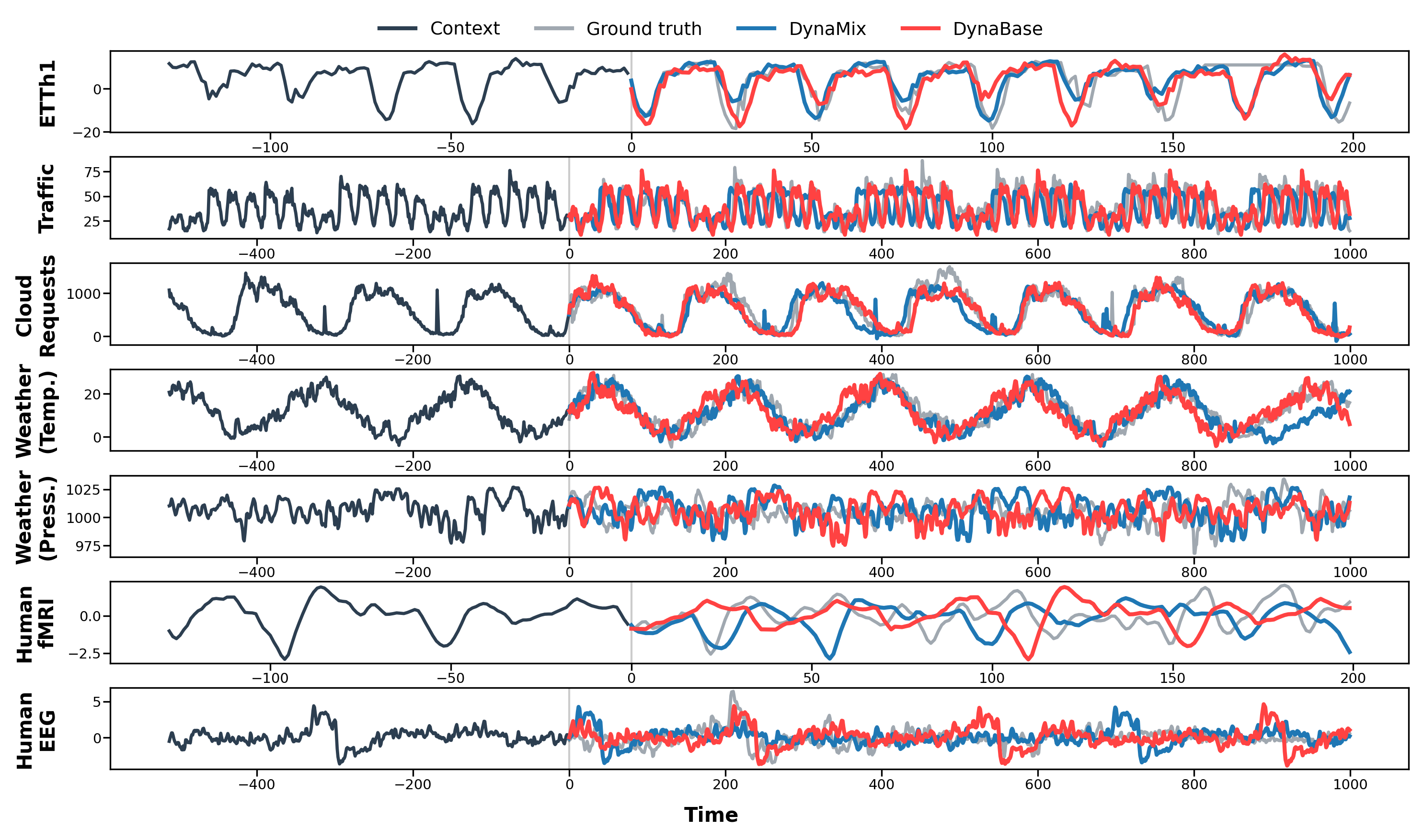}
    \caption{TS forecasts on $6d$ delay-embedded real-world time series for DynaBase (red) with $\alpha$ custom-trained, and DynaMix (blue). Context in dark gray.}
    \label{fig:TS_example}
\end{figure*}

\subsection{Loss landscape tracks dynamics}\label{sec:results-forecasting-strategy}

\begin{figure*}[!ht]
    \centering
    \includegraphics[width=0.95\linewidth]{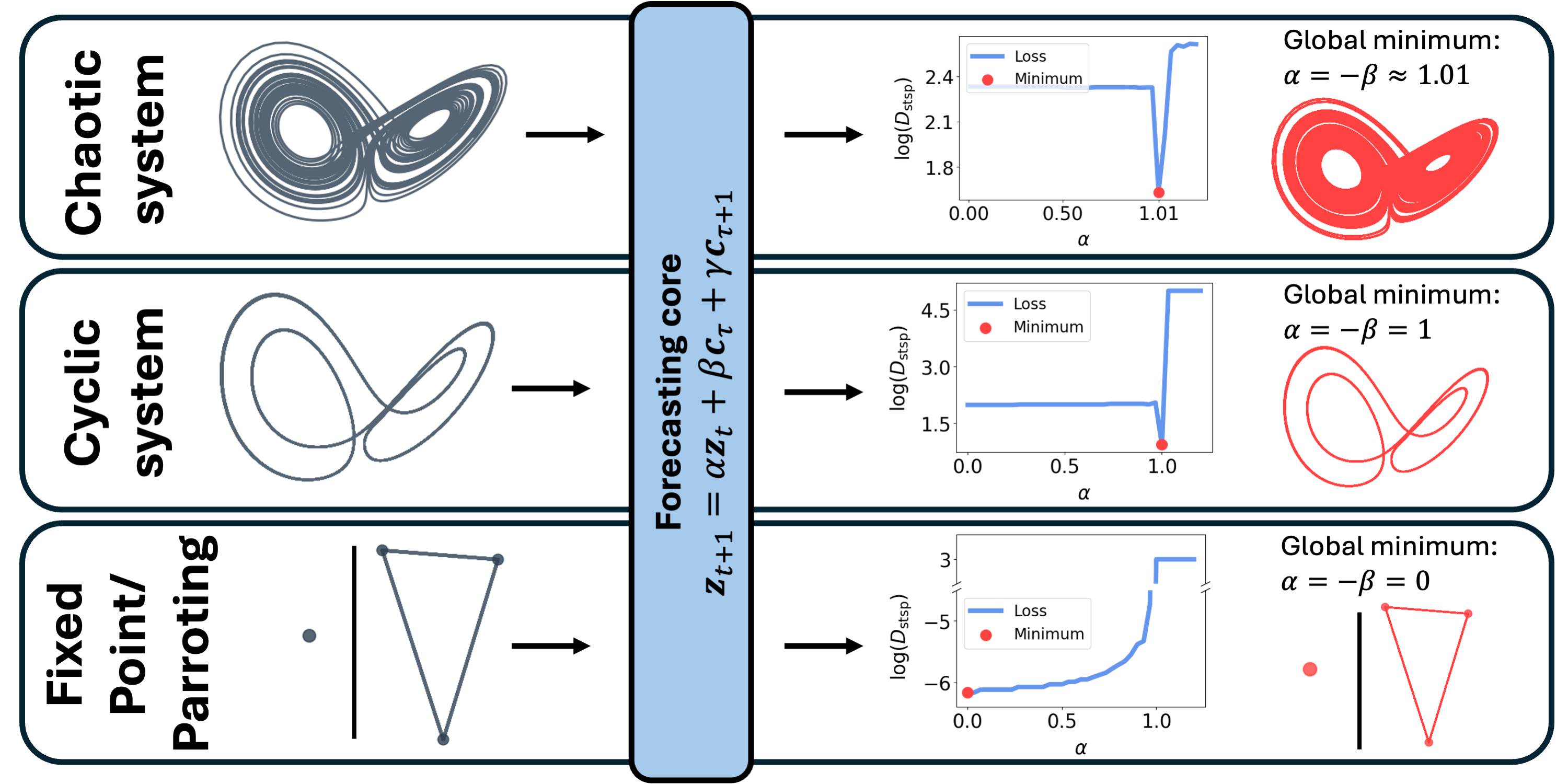}
    \caption{DynaBase' loss landscape reflects the dynamics of the target system when trained for DSR. \textbf{Top:} Lorenz-63 in chaotic regime (gray). The geometrical misalignment $D_{stsp}$ across $\alpha$ exhibits a global minimum (red dot) at $\alpha\approx1.01$, which produces a valid reconstruction of the attractor (red). \textbf{Center:} Lorenz-63 in cyclic regime. The $D_{stsp}$ minimum occurs at $\alpha=1$. \textbf{Bottom:} A fixed point or cyclic point (discrete cycle) (gray). $D_{stsp}$ is minimized at $\alpha=0$, corresponding to context parroting.}
    \label{fig:strategy}
\end{figure*}

To examine how DynaBase' parameters map onto the ground truth system's dynamical regime, we study the DSR loss landscape on a representative chaotic, cyclic and fixed point system (Fig. \ref{fig:strategy}). For these different dynamical regimes, the loss landscape exhibits distinct global minima which correspond to different dynamical mechanisms:
\begin{itemize}
    \item \textbf{Context parroting} $\alpha=-\beta=0$: the map reduces to $\bm{z}_{t+1}=\bm{c}_{\tau(\bm{z}_t)+1}$, i.e.\ the state snaps onto the successor of its nearest context neighbor. For a delay-embedded context signal, this exactly reproduces the context-parroting algorithm of \citet{zhang2026context} (see Thm. \ref{thm:parroting}). 
    \item \textbf{Flow mimicry} $\alpha=-\beta=1$: the map reduces to $\bm{z}_{t+1}-\bm{z}_t = \bm{c}_{\tau(\bm{z}_t)+1}-\bm{c}_{\tau(\bm{z}_t)}$, a forward-Euler step driven by the local finite difference $\dot{\bm{c}}\approx(\bm{c}_{\tau(\bm{z}_t)+1}-\bm{c}_{\tau(\bm{z}_t)})/h$ of the underlying vector field as sampled from the context. This yields a `zero-friction' direction in state space (neither con- nor divergence), as along the limit cycle ($\lambda_\text{max}=0$) of a continuous-time DS.
    \item \textbf{Chaos} $\alpha>1$: One obtains local divergence almost everywhere and, if bounded, produces chaotic trajectories (cf. Fig. \ref{fig:lorenz_alphas} for Lorenz-63 reconstructions along a bifurcation diagram).
\end{itemize}

\begin{figure*}[!ht]
    \centering
    \includegraphics[width=0.99\linewidth]{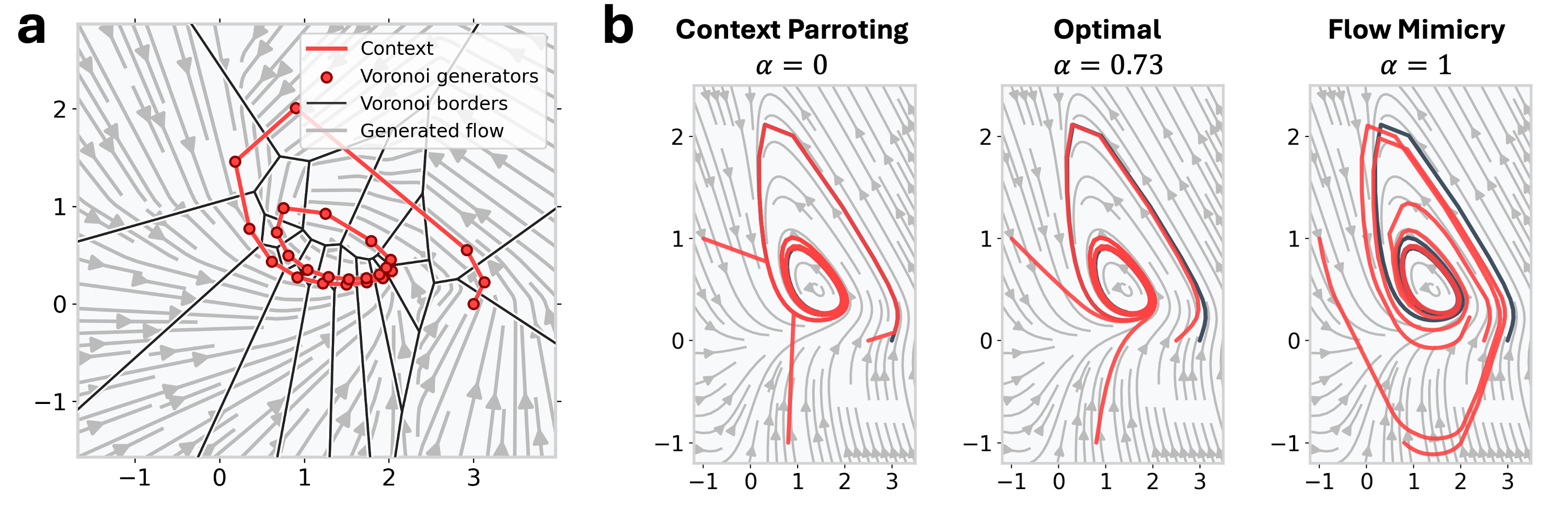}
    \caption{\textbf{a}) Voronoi tessellation for the periodic Selkov DS induced by the nearest-neighbor lookup $\tau(\cdot)$ in DynaBase. The context $\bm{C}_{1:T_C}$ partitions the state space into cells $V_s=\{\bm{z}:\tau(\bm{z})=s\}$, within which DynaBase is described by a strictly affine map $\bm{z}_{t+1}=\alpha\, \bm{z}_t + \tilde {\bm{c}}_s$. \textbf{b}) Out-of-context forecasting on the 2d Selkov system. When the initial condition lies outside of the context-covered region, neither parroting ($\alpha=0$, left) nor pure flow mimicry ($\alpha=1$, center) recovers the true vector field. Refitting $\alpha$ from a single additional short trajectory via eq. \ref{eq:linear-regression} produces a local flow estimate that aligns better with the true vector field (center).}
    \label{fig:selkov}
\end{figure*}

The nearest-neighbor lookup $\tau(\cdot)$ induces a Voronoi tessellation of the state space (Fig. \ref{fig:selkov}\textbf{a}). Inside each cell $V_s=\{\bm{z}:\tau(\bm{z})=s\}$, DynaBase is a strictly affine map $\bm{z}_{t+1}=\alpha\, \bm{z}_t + \tilde {\bm{c}}_s$, where $\tilde{\bm{c}}_s=\beta \bm{c}_s+\gamma \bm{c}_{s+1}$ is fixed by the cell, hence belongs to the class of well-studied switching piecewise-linear (PL) systems \cite{daafouz_stability_2002,sun_switched_2006}. Within each cell $V_s$, the model's Jacobian is simply given by $\alpha\bm{I}$, i.e. a single scalar $\alpha$ controls the local con- or divergence almost everywhere in state space. With $\alpha\approx 1.01$ just slightly above $1$ (Fig. \ref{fig:alpha-histogram}), DynaBase diverges almost everywhere but still stays sufficiently close to the cell centers to create a chaotic \textit{attractor}, despite not having a nonlinear activation function (in line with Thm. \ref{thm:chaos} \& Thm. \ref{thm:boundedness}). For context parroting with $\alpha=0$ (Thm. \ref{thm:parroting}), on the other hand, DynaBase snaps onto a cyclic point that perfectly traces out the empirically observed context trajectory. However, a cyclic point would result in this case even if the observed trajectory snippet comes from a chaotic system (Fig. \ref{fig:forecast_result_strategy}).

Pure context parroting also prevents generalization to new initial conditions of an observed DS, a feature of DynaMix demonstrated in \cite{hemmer2025true}. This is illustrated for the Selkov DS in Fig. \ref{fig:selkov}\textbf{b}: For $\alpha=0$, DynaBase strictly follows the context points and thus fails to map the vector field anywhere outside the context. On the other hand, pure context mimicry, $f_1$, while reproducing the $\lambda_\text{max}=0$ direction \textit{on} a limit cycle, fails to capture the convergence when started \textit{off} the stable cycle, since the finite difference $\bm{c}_{\tau(\bm{z}_t)+1}-\bm{c}_{\tau(\bm{z}_t)}$ is always evaluated at the nearest context point. Fitting $\alpha$ via eq. \ref{eq:KL_loss} on a trajectory bit that contains a transient, however, yields a better approximation of the surrounding vector field outside the observed context while still reproducing the cycle (Fig. \ref{fig:selkov}\textbf{b}, center).

\section{Conclusion}\label{sec:conclusion}
We reduced DynaMix, a mixture-of-experts DSR foundation model, to \emph{DynaBase}, a two-parameter nearest-neighbor affine recurrence that can be further reduced to a $1d$ family of maps by imposing a consistency condition. DynaBase exposes the minimal mechanisms necessary for zero-shot DSR by 1) closely following the context via its nearest-neighbor selector, thus generating TS with similar geometrical and temporal characteristics as represented in the data, while 2) regulating local divergence rates through a single parameter $\alpha$, thus enabling a spectrum of context-guided dynamical behaviors from point attraction ($\alpha<1$) and limit cycles ($\alpha=1$) to chaos ($\alpha>1$). Despite this minimal form, DynaBase matches or exceeds substantially larger architectures on both long-horizon DSR measures and short-term predictions.

Analysis of this minimal map and its training provides a number of interesting insights into the minimal requirements TS and DSR FMs may need to achieve zero-shot DS forecasts. First, $\alpha = 0$ exactly reproduces context parroting as observed in \cite{zhang2025zeroshotforecastingchaoticsystems,zhang2026context}, but this mechanism is inherently unable to reproduce chaotic or even true limit cycle behavior. Second, for $\alpha > 0$ this affine nearest-neighbor form admits the whole spectrum from convergence to fixed points or cycles to strictly chaotic activity for $\alpha > 1$, thus being sufficiently expressive to approximate a large range of DS. Third, for $\alpha \approx 1.01$ the map can even operate in zero-shot mode, at least on a common DSR benchmark set, exposing the profound role the context alone plays in structuring the forecast. Fourth, optimizing for ahead-predictions biases the map toward convergence to fixed or cyclic points, even when the underlying DS is chaotic, explaining why this training objective is suboptimal if DSR is the goal \cite{mikhaeil_difficulty_2022,hess_generalized_2023,platt2023constraining,jiang2023training}. More generally, with DynaBase we have established a formal framework that enables to study mechanisms of DSR in FMs and beyond in detail, creating an instrument for mathematical analysis of context-based forecasting and DSR methods. 

\paragraph{Limitations}
DynaBase is a minimal 2-parameter model. Although it highlights the fact that complex architectures with tens of thousands to many millions of parameters may not be needed for DSR and time series forecasting, the assumptions on which its reduction is based come at a cost in terms of the diversity of DS and vector field topologies that can be learned, as well as the accuracy with which geometrical properties of the vector field can be captured (as apparent from the remaining mismatch in Fig. \ref{fig:selkov}\textbf{b}). Currently, DynaBase may be viewed more as a research tool that allows for detailed mathematical analysis of context-based DSR FMs. A more in-depth study of DynaBase, e.g. its DS approximation capabilities or analytical computation of the full Lyapunov spectrum, which requires the concept of a saltation operator \cite{diBernardo2008,Coombes2024OscillatoryNetworksPWL}, is still pending.

\section*{Acknowledgements}
This work was supported by individual grants Du 354/15-1 (\# 502196519) and Du 354/18-1 (\# 567025973) from the German Research Foundation (DFG).


\bibliography{literature}
\bibliographystyle{plainnat}

\newpage
\clearpage
\appendix


\section{Appendix}

\subsection{Methodological details} \label{sec:method_details}

\paragraph{Training details}
Training of the DynaBase parameters $(\alpha, \beta)$ uses a sampled training sequence $\bm{X}_{\text{train}}\in\mathbb{R}^{N\times T}$ drawn from a DS, where the first $T_C$ steps serve as context signal $\bm{C}$. For the experiments in Fig. \ref{fig:performance}, we set $T = 2000$ and $T_C = 1000$. At evaluation time, we supplied the training sequence as context alongside the estimated parameter values $(\alpha, \beta)$, and assessed performance on a sequence unseen during training and context construction (using the self-consistent DynaBase $f_\alpha$ as defined in eq. \ref{eq:self-consistent-one-param}). This ensures that DynaBase receives exactly the same context and training sequence information as all other comparison methods.

To fit DynaBase via \textbf{linear regression}, we compute $(\alpha,\beta)$ in closed form via ordinary least-squares. Nearest neighbors $\tau(\cdot)$ are precomputed once via $\ell^2$-distances over $\bm{C}$ and the training sequence $\bm{X}$. To cast this as a standard linear regression problem, we define the response vector $\bm{b}$ and design matrix $\bm{A}$ based on eq. \ref{eq:dynabase}, for each training sample $i$
\begin{align}
    \bm{b}_i &= \bm{z}_{i+1} - \bm{c}_{\tau(\bm{z}_i)+1}, \\
    \bm{A}_i &= \Bigl[\;\bm{z}_i - \bm{c}_{\tau(\bm{z}_i)+1}\;,\;\bm{c}_{\tau(\bm{z}_i)} 
    - \bm{c}_{\tau(\bm{z}_i)+1}\;\Bigr].
\end{align}
The regression then takes the form $\bm{b} = \bm{A}[\alpha, \beta]^\top$, and the coefficients are obtained by the standard least-squares solution
\begin{equation}
    [\alpha, \beta]^\top = (\bm{A}^\top \bm{A})^{-1}\bm{A}^\top \bm{b}.
\end{equation}
In the self-consistent setting of eq. \ref{eq:self-consistent-one-param}, the constraint $\beta=-\alpha$ reduces the two-parameter fit to a single unknown scalar, simplifying $\bm{A}_i = \bm{z}_i - \bm{c}_{\tau(\bm{z}_i)}$.

To train parameters via \textbf{grid search}, we sweep across $\alpha \in [0,1.2]$ with step size $\Delta\alpha=0.002$ along the self-consistency line $\beta=-\alpha$. For each candidate $\alpha$, DynaBase is rolled out for $T=10{,}000$ steps using the context $\bm{C}$ and scored by the KL-divergence loss given in eq. \ref{eq:KL_loss} between predicted trajectories $\hat{\bm{X}}$ and the training data. The global minimizer $\alpha$ is returned.

Training and evaluation of DynaBase was performed on a single CPU (AMD EPYC 9655 96-Core).

\paragraph{Comparison methods}
We compare DynaBase against three groups of baselines, all evaluated on the same $54$ test DS and the same input data of length $T=2000$ as in \cite{hemmer2025true}.

\emph{Zero-shot foundation models:} DynaMix \cite{hemmer2025true}, the DSR FM that DynaBase is reduced from (mixture-of-experts of AL-RNNs with context-dependent gating; pretrained on $\sim\!6\times 10^5$ trajectories from $34$ DS); Panda \cite{lai2025panda}, a zero-shot DS forecasting model trained on simulated DS data; Chronos-T5-base \cite{ansari2024chronos}, a transformer FM that quantizes real-valued series into discrete tokens and forecasts autoregressively on a corpus of synthetic and real-world time series; and Chronos-2 \cite{ansari2025chronos2univariateuniversalforecasting}, its multivariate extension augmenting the T5 backbone with a group-attention mechanism. Evaluations performed as in \cite{hemmer2025true}.

\emph{Custom-trained DSR models:} An AL-RNN \cite{brenner_almost_2024} (eq. \ref{eq:alrnn}), which is used for the experts in DynaMix, a Neural-ODE \cite{chen_neural_2018}, and a reservoir computer \cite{patel_using_2023}; each trained from scratch on the same $T=2000$ context segment provided to the foundation models. Evaluations as in \cite{hemmer2025true}.

\emph{Context parroting:} The algorithm of \citet{zhang2026context}, which retrieves the best-matching motif in $\bm{C}$ and copies its successor as forecast. As proven in Thm. \ref{thm:parroting}, this baseline corresponds to the $\alpha=0$ case of DynaBase on a delay-embedded context.

\subsection{Performance measures}\label{appx:metrics}

\paragraph{Long-term metrics}
To evaluate the \textit{geometric similarity} between true and model-generated reconstructions, we employ a measure, $D_{\text{stsp}}$, based on the Kullback-Leibler (KL) divergence evaluated in state space \cite{koppe_identifying_2019, mikhaeil_difficulty_2022, gilpin_model_2023, pals2024inferring, zhang2025zeroshotforecastingchaoticsystems,lai2025panda,hemmer2025true}. This measure quantifies the (mis)match between the ground-truth spatial distribution of trajectory points, $p_{\text{true}}(\bm{x})$, and the distribution $p_{\text{gen}}(\bm{x}|\bm{z})$ of points from trajectories freely generated by the model:
\begin{equation}
    D_{\text{stsp}}\bigl(p_{\text{true}}(\bm{x}) \,\|\, p_{\text{gen}}(\bm{x}|\bm{z})\bigr) = \int p_{\text{true}}(\bm{x}) \log \frac{p_{\text{true}}(\bm{x})}{p_{\text{gen}}(\bm{x}|\bm{z})} \, d\bm{x}.
\end{equation}
In practice, we approximate this via a discrete binning of the state space into $K = m^N$ bins, where $m$ is the number of bins per dimension and $N$ is the system dimensionality, estimating occupation probabilities through relative frequency of visits $\hat{p}_i$:
\begin{equation}
    D_{\text{stsp}} \approx \sum_{i=1}^{K} \hat{p}_{\text{true};i} \log\frac{\hat{p}_{\text{true};i}}{\hat{p}_{\text{gen};i}}.
\end{equation}
To ensure that the system has reached a steady-state distribution, long trajectories ($T=10000$) are sampled from the trained model. For the 3d attractors we set $m = 30$ bins per dimension, following \cite{hemmer_optimal_2024}.

To evaluate long-term \textit{temporal} reconstruction quality, we calculate the Hellinger distance $D_H$ between the power spectra of the ground-truth and model-generated time series \cite{mikhaeil_difficulty_2022, hess_generalized_2023, pals2024inferring, durstewitz2026position}. We first apply dimension-wise Fast Fourier Transforms (FFT) to the time series and then smooth the resulting power spectra with a Gaussian kernel (with $\sigma=20$), followed by normalization to enable comparison across dimensions. High-frequency tails dominated by noise are removed. The Hellinger distance between the smoothed true spectrum $F(\omega)$ and the generated spectrum $G(\omega)$ is given by
\begin{equation}
    D_H\bigl(F(\omega), G(\omega)\bigr) = \sqrt{1 - \int_{-\infty}^{\infty} \sqrt{F(\omega)\,G(\omega)}\, d\omega} \;\in [0, 1],
\end{equation}
where values close to $0$ indicate high spectral agreement and values close to $1$ indicate complete disagreement. $D_H$ is reported as the average across all dimensions. Note that both $D_{\text{stsp}}$ and $D_H$ are designed to capture \emph{long-term} reconstruction quality in the limit $T \to \infty$, after transients have settled and the attractor geometry has been traced out reasonably well (which was the case here for $T=10000$ steps).

\paragraph{Short-term forecast quality}
For evaluating short-term predictive performance, we use the Mean Absolute Scaled Error (MASE) \cite{hewamalage2023forecast, hemmer2025true}, which is a normalized $n$-step-ahead prediction error aggregated across all $N$ output dimensions:
\begin{equation}
    \text{MASE} = \frac{1}{N} \sum_{i=1}^{N} \frac{\dfrac{1}{n} \displaystyle\sum_{t=1}^{n} \bigl|x_{i,t} - \hat{x}_{i,t}\bigr|}{\dfrac{1}{T} \displaystyle\sum_{t=1}^{T} \bigl|x_{i,t} - x_{i,t-1}\bigr|},
\end{equation}
where $T$ denotes the total length of the ground truth time series, $n$ is the forecast horizon, $x_{i,t}$ the true value, and $\hat{x}_{i,t}$ the model prediction for dimension $i$ at time $t$. We evaluate at $n = 10$ steps.

\paragraph{Lyapunov exponent}\label{sec:lyapunov}
To numerically estimate the \textit{maximum Lyapunov exponent} $\lambda_{\text{max}}$ from a time series $\bm{X}$, we use the \textit{Rosenstein algorithm} \cite{rosenstein1993practical}. For each state $\bm{x}_i$, its nearest neighbor $\bm{x}_j$ is identified subject to $|i - j| > l_t$, which excludes direct temporal neighbors on the same trajectory segment. We set $l_t = 100$ (based on the mean periodicity \cite{rosenstein1993practical}). The divergence between neighboring trajectories is then tracked over time steps $k = 0, \ldots, k_{\max}$,
\begin{equation}
    d_i(k) = \| \bm{x}_{i+k} - \bm{x}_{j(i)+k} \|_2\;.
\end{equation}
Since nearby trajectories diverge as
\begin{equation}
    d_i(k) \approx d_i(0)\, e^{\lambda_{\text{max}} k \Delta t}\;,
\end{equation}
where $\Delta t$ is the sampling interval, taking the logarithm and averaging over all pairs gives
\begin{equation}
    \langle \ln d(k) \rangle \approx \lambda_{\max} k \Delta t + C\;.
\end{equation}
For chaotic systems, this quantity grows approximately linearly with $k$ for a suitable $k_{\max}$ (here: $100$) \cite{kantz_nonlinear_2004}, with slope $\lambda_{\max} > 0$, whereas stable periodic systems yield $\lambda_{\max} = 0$.

When the equations of motion $\dot{\mathbf{x}} = f(\mathbf{x})$ are known, the full Lyapunov spectrum $\{\lambda_i\}$ can be computed via a \textit{QR decomposition-based algorithm} \cite{benettin1980lyapunov}, which integrates the system along its variational equations $\dot{Q} = J(\mathbf{x})Q$, where $J(\bm{x}) = \partial f / \partial \bm{x}$ is the Jacobian of $f$ and $Q$ is an $n \times n$ matrix of perturbation vectors initialized to the identity $I_n$. At regular time intervals, a QR decomposition $Q = \tilde{Q}R$ is performed and the columns of $\tilde{Q}$, which form an orthonormal basis, replace $Q$. The diagonal entries $R_{ii}$ record the local expansion rates, and the $k$-th Lyapunov exponent is obtained as the time-averaged logarithmic growth rate,
\begin{equation}
    \lambda_k = \lim_{T \to \infty} \frac{1}{T} \sum_{j} \log \left| R_{kk}^{(j)} \right|,
\end{equation}
where the sum runs over all QR factorizations performed up to time $T$. If the governing equations are not available, the full spectrum can instead be estimated numerically from the time series using the method of \textit{Sano and Sawada} \cite{sano1985measurement,eckmann1986liapunov}. For each state $\bm{x}_i$, its $n_{\text{nn}}$ nearest neighbors $\bm{x}_j$ (for the Lorenz-96 experiment $n_{\text{nn}}=100$) are used to fit a local Jacobian $\hat{J}_i$ by least squares based on how the displacements evolve over one time step,
\begin{equation}
    \bm{x}_{j+1} - \bm{x}_{i+1} \approx \hat{J}_i\,(\bm{x}_j - \bm{x}_i)\;.
\end{equation}
The estimated Jacobians are then used to propagate the perturbation vectors, $Q \leftarrow \hat{J}_i Q$, followed by the same QR re-orthonormalization as above.

\subsection{Details on DynaMix reduction}\label{appx:ablation}

\begin{table}[h]
\centering
\caption{Zero-shot performance of DynaMix and model simplifications performed across all 54 test DS used in \cite{hemmer2025true}. DynaMix and simplified models are compared for the same context length $T_C=2000$. Reported are median$\pm$MAD for geometrical disagreement $D_\text{stsp}$, temporal disagreement $D_H$, and 10-step-ahead prediction error $\text{MASE}$.}
\vspace{0.5em}
\begin{tabular}{lccc}
\hline
Model & $D_{\text{stsp}}\downarrow$ & $D_H\downarrow$ & $\text{MASE}\downarrow$ \\
\hline
Original DynaMix & $3.83\pm1.19$ & $0.17\pm0.08$ & $0.35\pm0.17$ \\
Linear experts \& gating ($M>N$) & $3.10\pm1.19$ & $0.13\pm0.07$ & $0.26\pm0.11$ \\
Linear experts \& gating ($M=N$) & $2.85\pm1.30$ & $0.13\pm0.06$ & $0.23\pm0.10$ \\
MLP approximation of bilinear term & $3.19\pm1.33$ & $0.14\pm0.07$ & $0.34\pm0.14$ \\
Pre-trained DynaMix with $\tau_{\text{att}} \to 0$ & $2.75\pm1.05$ & $0.12\pm0.06$ & $0.40\pm0.20$ \\
\hline
\end{tabular}\label{tab:dynamix_ablations}
\end{table}

Here we provide the supporting evidence and derivation details for the successive reduction of DynaMix to the \emph{recursive affine NN map} eq. \ref{eq:linear-general} described in Sec.~\ref{sec:ablations}.

\paragraph{Linear approximation}
Starting from eq.~\ref{eq:update-rule-dynamix}, DynaMix is a weighted sum of piecewise-linear experts $F_j^\text{exp}$ with time-dependent expert weights $\bm{w}^{exp}_t$. To simplify the architecture we linearize both parts. Specifically, each expert is an almost-linear RNN (eq.~\ref{eq:alrnn}) with piecewise-linear activation $\Phi^*$ on $P$ units and linear mapping for the remaining $M-P$. If we replace $\Phi^*$ with the identity, each expert becomes an affine function of $\bm{z}_t$. The expert weights $\bm{w}^{exp}_t$ (eq.~\ref{eq:exp_weights}) are nonlinear due to the MLP and a softmax output. As previously observed in \cite{hemmer2025true}, the MLP can be replaced by a linear layer without loss in performance. Moreover, across the relevant domain, the softmax $\sigma$ can be approximated by an affine mapping as well. This makes each expert affine in $\bm{z}_t$ and each weight affine in $(\bm{z}_t,\tilde{\bm{C}}\bm{w}^{att}_t)$, so that their product is bilinear in $\bm{z}_t$. Table \ref{tab:dynamix_ablations} (\emph{Linear experts \& gating, $M>N$}) empirically confirms that these approximations leave performance on the chaotic benchmark set unchanged.

\paragraph{Latent dimension}
The experts in DynaMix operate in an $M$-dimensional latent space distinct from the $N$-dimensional observation space, requiring the projection $\bm{D}$ in eq. \ref{eq:attn_weights}. Replacing the experts with linear maps removes the need for such a distinction since no nonlinear latent units are required, yielding $M=N$ and $\bm{D}=\bm{I}$, such that the model operates entirely in observation space. Table \ref{tab:dynamix_ablations} (\emph{Linear experts \& gating, $M=N$}) confirms this does not reduce performance on the present test set. Working in a single space is also a prerequisite for the $O(N)$ equivariance exploited in Thm. \ref{thm:uniqueness}, since rotations and the nearest-neighbor distance must act on the same space in which the state $\bm{z}_t$ and the context $\bm{C}$ live. Note we can always perform a classical delay-embedding \cite{takens_detecting_1981,sauer_embedology_1991} should this dimensionality not be sufficient to properly represent an empirically observed DS.

\paragraph{Universal approximation of the bilinear term}
We exploit the universal approximation properties of ReLU-based MLPs \cite{yarotsky2017error} to approximate the bilinear form in $\bm{z}_t$ and $\bm{w}^{exp}$. Ablations confirm that the MLP can, in turn, itself be reduced to a linear layer (see Table \ref{tab:dynamix_ablations}), leaving $\tilde{\bm{C}}=\text{CNN}(\bm{C})$ and $\bm{w}^{att}_t(\bm{z}_t, \bm{C})$ as the only nonlinearities. This yields the recurrent linearized map form eq. \ref{eq:dynamix_linear}.

\paragraph{Nearest-neighbor selector}
In the publicly available DynaMix checkpoint, the learned attention temperature $\tau_\text{att}$ converged to $|\tau_{\text{att},\text{learned}}| \approx 0.06$, indicating that the soft attention eq. \ref{eq:attn_weights} is already operating close to its $\tau_\text{att} \to 0$ limit, i.e. effectively selecting a single nearest context point at each step. Two deliberate departures from DynaMix are made in passing to the hard limit. First, DynaMix injects Gaussian exploration noise $\varepsilon \sim \mathcal{N}(0, \bm{\Sigma})$ into the attention eq. \ref{eq:attn_weights}; in favor of a minimal and deterministic model, we set $\bm{\Sigma}=0$. Second, DynaMix computes context distances with the $\ell^1$-norm, whereas we adopt the Euclidean $\ell^2$-norm. The invariance of $\ell^2$ distances under orthogonal transformations $R \in O(N)$, is used in Thm. \ref{thm:uniqueness}. Together these choices yield the deterministic selector $\tau(\bm{z}_t) = \argmin_s \|\bm{c}_s - \bm{z}_t\|_2$.

\paragraph{CNN expansion}
To make the recursive map fully explicit in terms of individual context points, we expand the CNN operation. Since DynaMix uses a single linear CNN layer with kernel size $2$, for $s \in 1\dots T_C - 1$ we can write
\begin{equation}
    \tilde{\bm{c}}_s = \text{CNN}(\bm{C})_s = \bm{K}^{(1)} \bm{c}_s + \bm{K}^{(2)} \bm{c}_{s+1} + \bm{b},
\end{equation}
where $\bm{K}^{(1)}, \bm{K}^{(2)} \in \mathbb{R}^{N \times N}$ and $\bm{b} \in \mathbb{R}^N$ are learned weights. Substituting into eq. \ref{eq:dynamix_linear} with the nearest-neighbor selector and absorbing $\bm{B}\bm{K}^{(1)}$, $\bm{B}\bm{K}^{(2)}$ and the bias into $\bm{B}^{(1)}$, $\bm{B}^{(2)}$, and $\tilde{\bm{h}}=\bm{h}+\bm{B}\bm{b}$ respectively, yields the \emph{recursive affine NN map} eq. \ref{eq:linear-general} as stated in the main text.

\subsection{Datasets} \label{sec:data}
Our DS dataset consists of simulated time series of length $10^5$, produced from $3d$ DS collected in \cite{gilpin_chaos_2022}. For the test set, we used the same $54$ distinct systems as in \cite{hemmer2025true} to enable a fair comparison among all the baseline models in relation to the results published originally. To evaluate DynaBase in a zero-shot setting, we estimated $\alpha$ using a training set of $30$ different chaotic DS, analogous to the setup in \cite{hemmer2025true}.
\begin{figure*}[!ht]
    \centering
    \includegraphics[width=0.6\linewidth, angle=90]{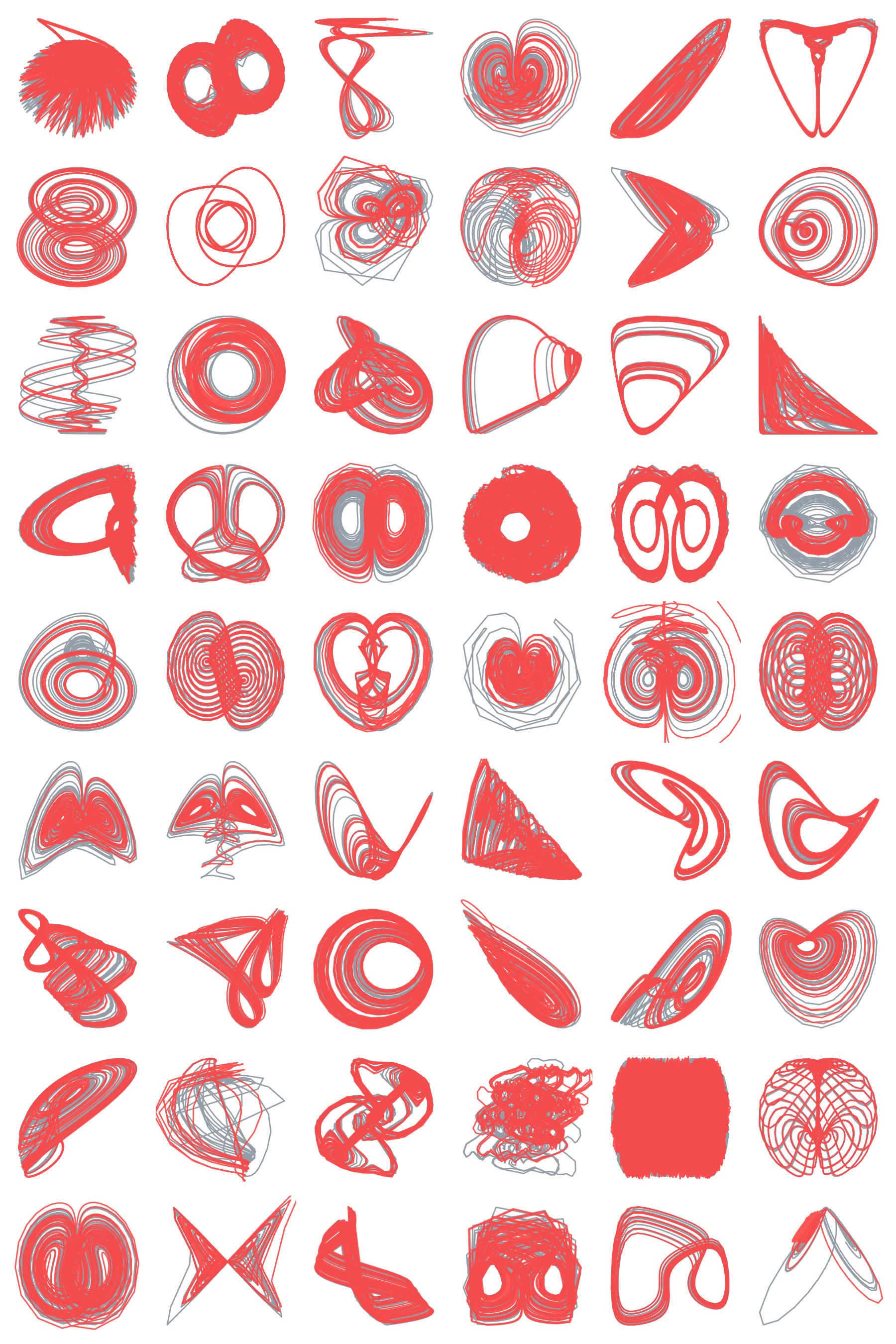}
    \caption{DSR test set of 54 different DS gathered from \cite{gilpin_chaos_2022}. Ground truth (gray) and DynaBase reconstructions using a context window of $T_C=2000$ (red).}
    \label{fig:all_reconstructions}
\end{figure*}

In addition, we conduct experiments on the $2d$ Selkov system \cite{sel1968self}, which models the kinetics of an open monosubstrate enzyme reaction and is defined by the equations
\begin{equation}
    \begin{aligned}
    \frac{dx}{dt} &= -x + ay + x^2y, \\
    \frac{dy}{dt} &= b -ay - x^2y ,
    \end{aligned}
\end{equation}
where we chose $a=0.1$ and $b=0.5$. The system was solved numerically with integration time step $\Delta t = 0.3$ using \texttt{scipy.integrate} with the \texttt{RK45} solver.

For Fig. \ref{fig:lorenz_alphas} we used the Lorenz-63 attractor \cite{lorenz_deterministic_1963}, originally formulated by Edward Lorenz in 1963 \cite{lorenz_deterministic_1963} to model atmospheric convection. It is defined by
\begin{align}
    \frac{\text{d}x}{\text{d}t} & = \sigma(y-x)\\ \nonumber
    \frac{\text{d}y}{\text{d}t} & = x(\rho-z)-y\\ \nonumber
    \frac{\text{d}z}{\text{d}t} & = xy-\beta z,
\end{align}
where $\sigma, \rho, \beta$, are control parameters of the system (we fix $\sigma=10$ and $\rho=28$, and vary $\beta\in[0,1]$, producing regimes with point, cyclic, and chaotic attractors). The system was solved numerically with integration time step $\Delta t = 0.02$ using \texttt{scipy.integrate} with the \texttt{RK45} solver and was standardized afterward.

To test DynaBase on a higher-dimensional system with multiple time scales, we use the two-scale Lorenz-96 model \cite{lorenz1996predictability}, in which slow ($X_k$) and fast ($Y_k$) variables interact. We use $K=5$ slow and corresponding fast variables, resulting in an $N=10$ dimensional system,
\begin{equation}
    \begin{aligned}
    \frac{\text{d}X_k}{\text{d}t} &= -X_{k-1}(X_{k-2}-X_{k+1}) - X_k + F - \frac{hc}{b}\,Y_k, \\
    \frac{\text{d}Y_k}{\text{d}t} &= -cb\,Y_{k+1}(Y_{k+2}-Y_{k-1}) - c\,Y_k + \frac{hc}{b}\,X_k,
    \end{aligned}
\end{equation}
with cyclic indices $k=1,\dots,K$. We set $F=10$, $h=1$, $c=20$, and $b=10$, yielding a hyperchaotic system with two positive Lyapunov exponents, $\lambda_{\max}\approx 3.6$. The system was solved numerically with a fixed-step fourth-order Runge-Kutta scheme with $\Delta t = 0.02$, and standardized per dimension.

For evaluation on empirical (real-world) TS, we adopt the same TS datasets and evaluation protocol as in \cite{durstewitz2026position}, spanning multiple domains such as weather signals, traffic, human physiological signals, cloud requests, and the ETTh1 dataset. All data is publicly available and links can be found in \cite{durstewitz2026position} (ETTh1: \url{https://github.com/zhouhaoyi/ETDataset}; Traffic: \url{https://www.kaggle.com/datasets/fedesoriano/traffic-prediction-dataset/data}: Cloud Requests: \url{https://github.com/sir-lab/data-release}; Weather: \url{https://www.dwd.de/EN/ourservices/cdc/cdc_ueberblick-klimadaten_en.html}; fMRI: \cite{kramer22a}; EEG: \cite{schalk_bci2000_2004}).

\clearpage
\subsection{Further empirical results}

\begin{figure*}[!ht]
    \centering
    \includegraphics[width=0.97\linewidth]{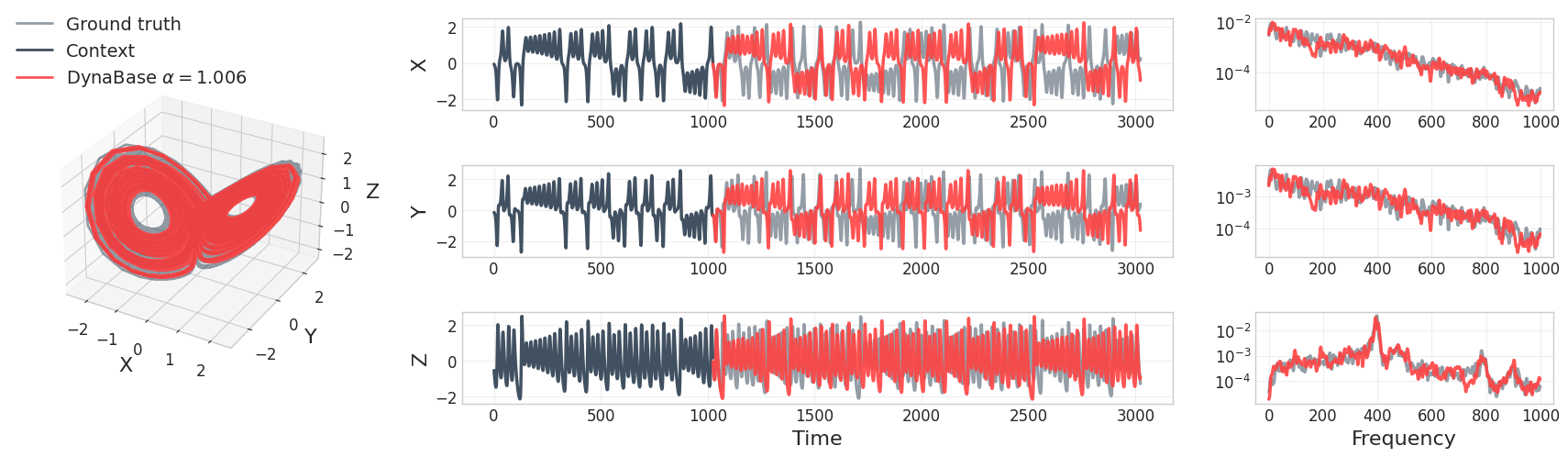}
    \caption{Zero-shot reconstruction of chaotic Lorenz-63 system using DynaBase with $\alpha=1.006$. Left: State space, center: time graphs, right: power spectrum.}
    \label{fig:reconstruction1}
\end{figure*}
\begin{figure*}[!ht]
    \centering
    \includegraphics[width=0.97\linewidth]{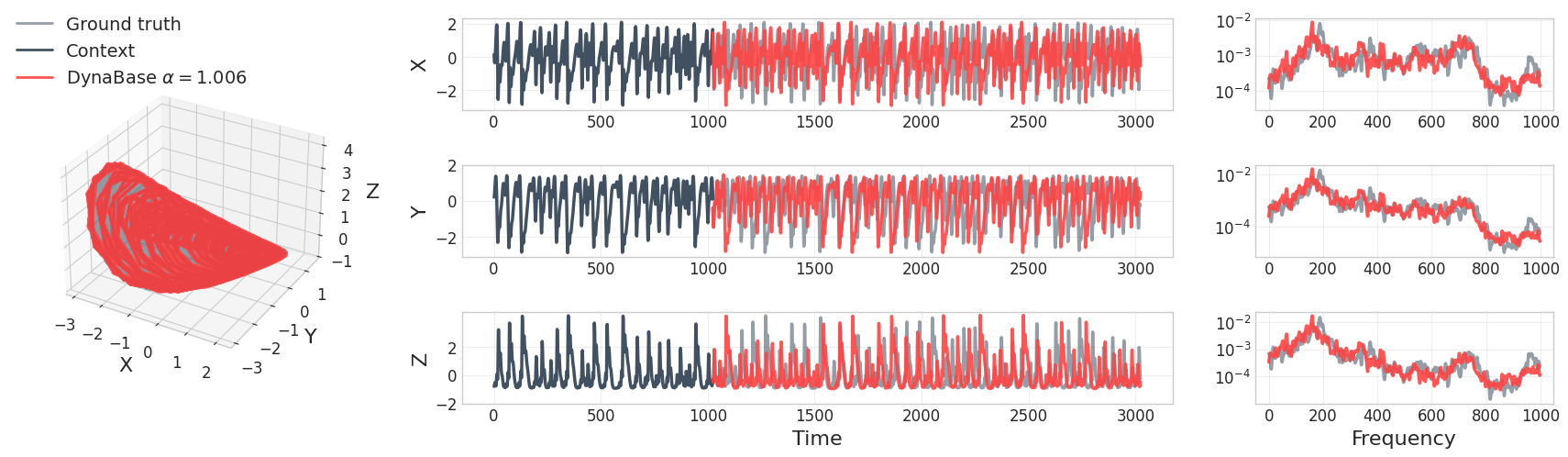}
    \caption{Zero-shot reconstruction of chaotic Sprott-F system using DynaBase with $\alpha=1.006$. Left: State space, center: time graphs, right: power spectrum.}
    \label{fig:reconstruction2}
\end{figure*}
\begin{figure*}[!ht]
    \centering
    \includegraphics[width=0.97\linewidth]{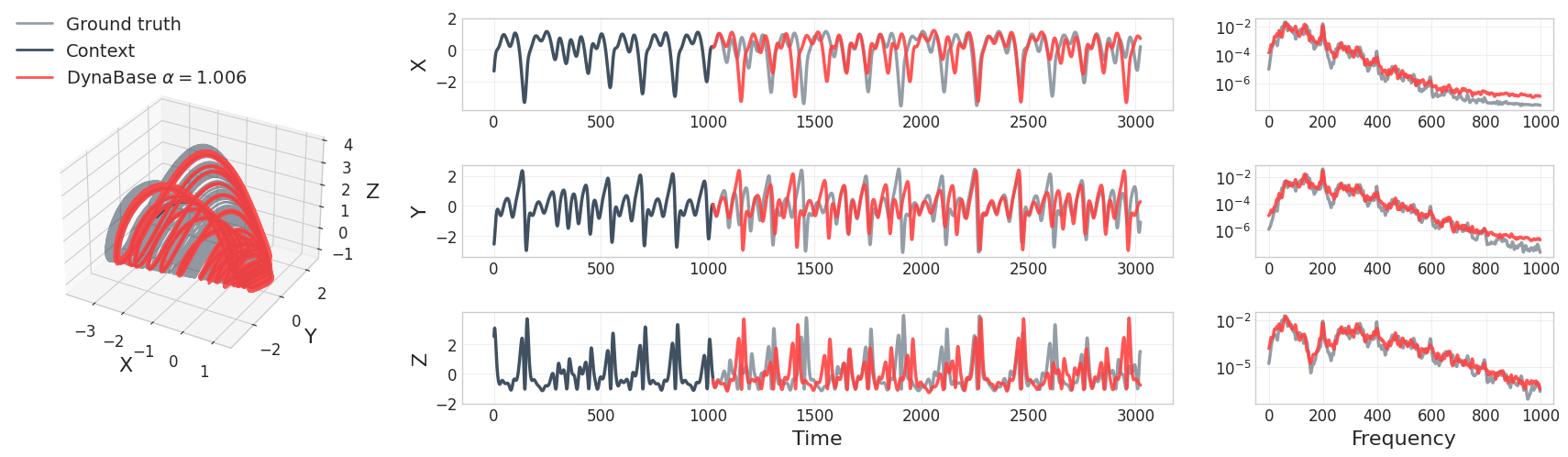}
    \caption{Zero-shot reconstruction of chaotic Sprott-D system using DynaBase with $\alpha=1.006$. Left: State space, center: time graphs, right: power spectrum.}
    \label{fig:reconstruction3}
\end{figure*}
\begin{figure*}[!ht]
    \centering
    \includegraphics[width=0.97\linewidth]{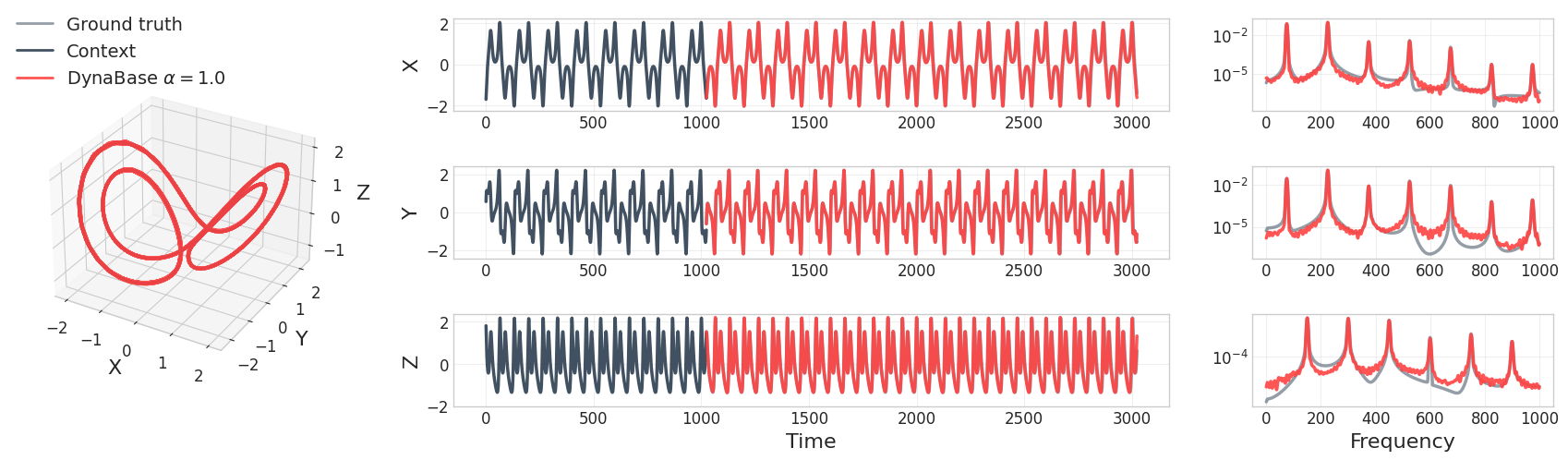}
    \caption{Zero-shot reconstruction using $\alpha=1$ of cyclic Lorenz-63 system using DynaBase. Left: State space, center: time graphs, right: power spectrum.}
    \label{fig:reconstruction4}
\end{figure*}

\begin{figure*}[!ht]
    \centering
    \includegraphics[width=0.7\linewidth]{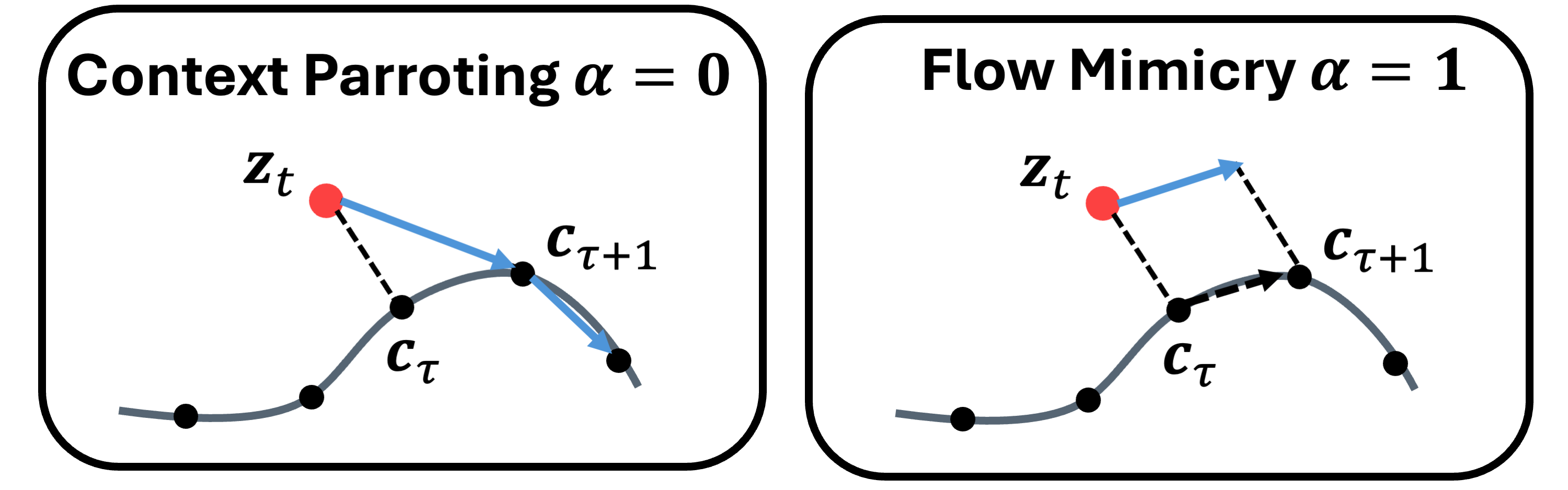}
    \caption{Illustration of context parroting and flow mimicry. \textbf{Left:} \emph{context parroting} --- at $(\alpha,\beta)=(0,0)$ the state snaps onto the successor $\bm{c}_{\tau(\bm{z}_t)+1}$ of its nearest context neighbor, producing a piecewise-constant map that replays the context. \textbf{Right:} \emph{flow mimicry} --- at $(\alpha,\beta)=(1,-1)$ the state is updated by adding the finite-difference increment $\bm{c}_{\tau(\bm z_t)+1}-\bm{c}_{\tau(\bm z_t)}$ evaluated at the nearest context neighbor, a local Euler step on the context vector field. }
    \label{fig:strategy_illustration}
\end{figure*}

\begin{figure*}[!ht]
    \centering
    \includegraphics[width=0.99\linewidth]{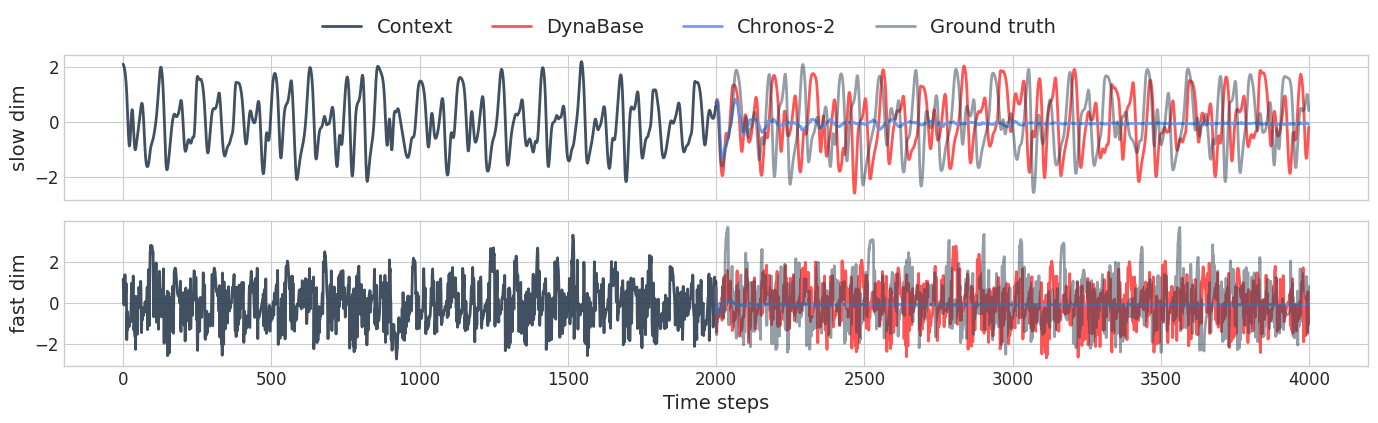}
    \caption{Prediction of the 10d, hyperchaotic, multiscale Lorenz-96 by DynaBase ($\alpha=1.006$) and Chronos-2. Top shows a slowly and bottom a quickly evolving variable.}
    \label{fig:multiscale_lorenz}
\end{figure*}

\begin{table}[!ht]
\centering
\caption{Zero-shot performance of DynaBase and Chronos-2 on the 10d, hyperchaotic, multiscale Lorenz-96 system for context length $T_C=2000$. Reported are geometrical disagreement $D_\text{stsp}$ (due to the high dimensionality, the distributions were estimated via Gaussian mixture models as in \cite{hess_generalized_2023}), temporal disagreement $D_H$, and 10-step-ahead prediction $\text{MASE}$.}
\vspace{0.5em}
\begin{tabular}{lccc}
\hline
Model & $D_{\text{stsp}}\downarrow$ & $D_H\downarrow$ & $\text{MASE}\downarrow$ \\
\hline
DynaBase & $0.19$ & $0.12$ & $1.01$ \\
Chronos-2 & $6.89$ & $0.89$ & $0.69$ \\
\hline
\end{tabular}\label{tab:lorenz96_multiscale}
\end{table}

\begin{figure*}[!ht]
    \centering
    \includegraphics[width=0.99\linewidth]{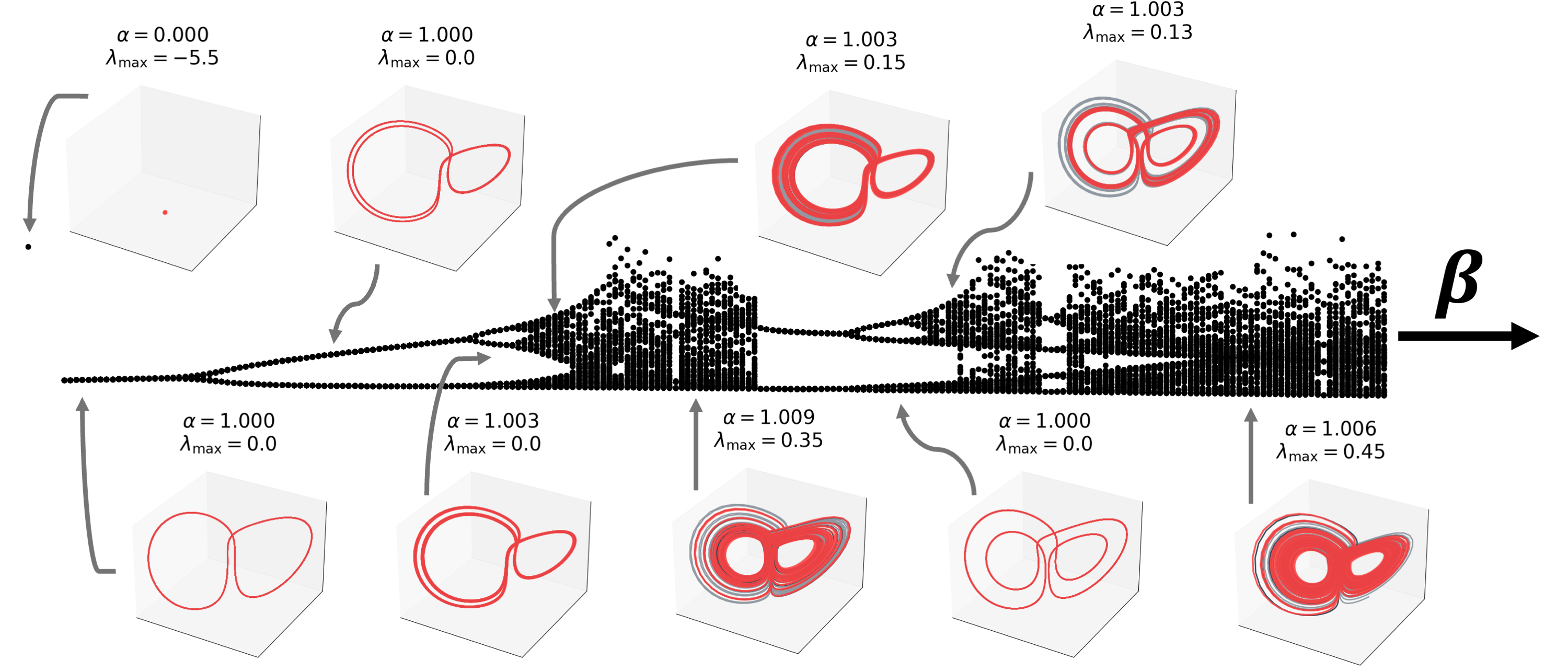}
    \caption{Different reconstructions by DynaBase along a bifurcation diagram of the Lorenz-63 system with system parameters $\sigma=10$, $\rho=28$ and $\beta\in[0,1]$. For each reconstruction, the estimated alpha value together with the GT system's true Lyapunov exponent (estimated via QR decomposition, see Sec. \ref{sec:lyapunov}) is indicated.}
    \label{fig:lorenz_alphas}
\end{figure*}

\begin{figure*}[!ht]
    \centering
    \includegraphics[width=0.99\linewidth]{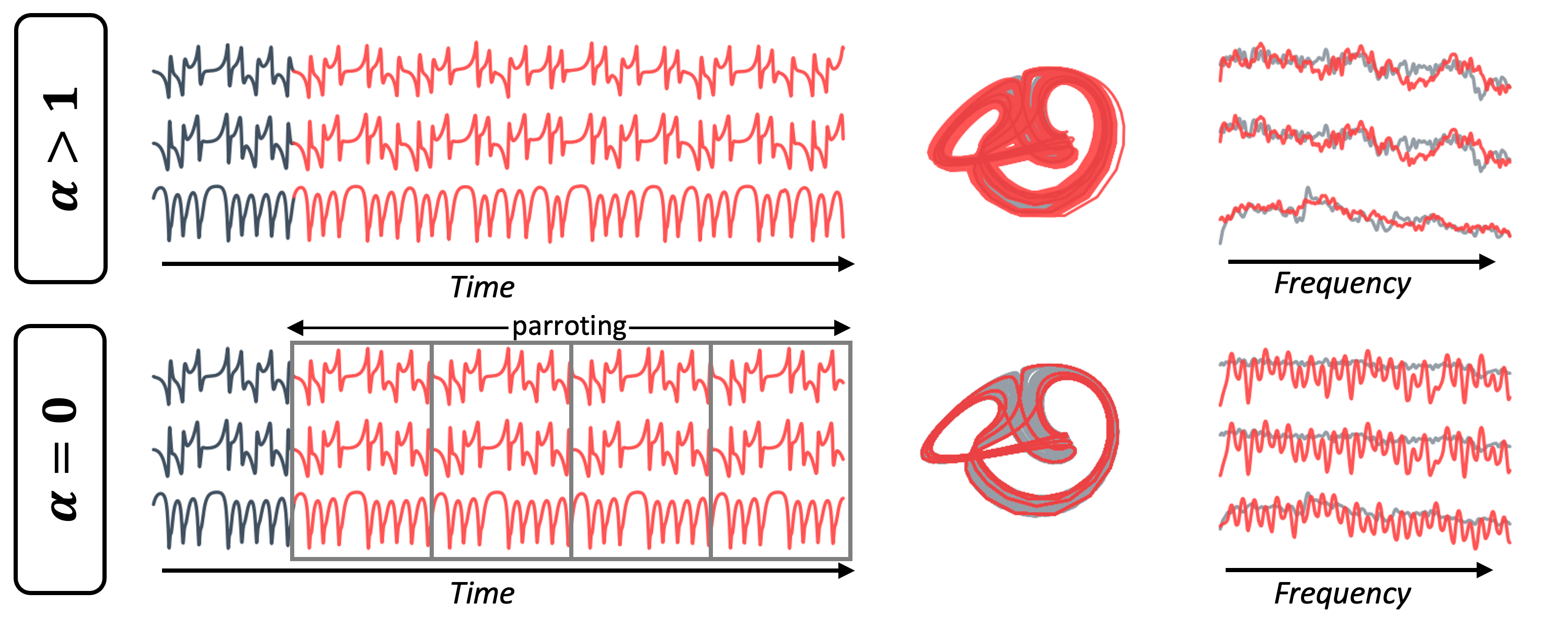}
    \caption{DynaBase forecasting of a chaotic laser system \cite{gilpin_chaos_2022}, using the same context once with $\alpha>1$ and once with $\alpha=0$. Note that despite the system's non-periodic context trajectory, parroting ($\alpha=0$) incorrectly produces a discrete cycle. Shown are time graphs (left), state space (center), and power spectrum (right).}
    \label{fig:forecast_result_strategy}
\end{figure*}

\begin{figure*}[!ht]
    \centering
    \includegraphics[width=0.99\linewidth]{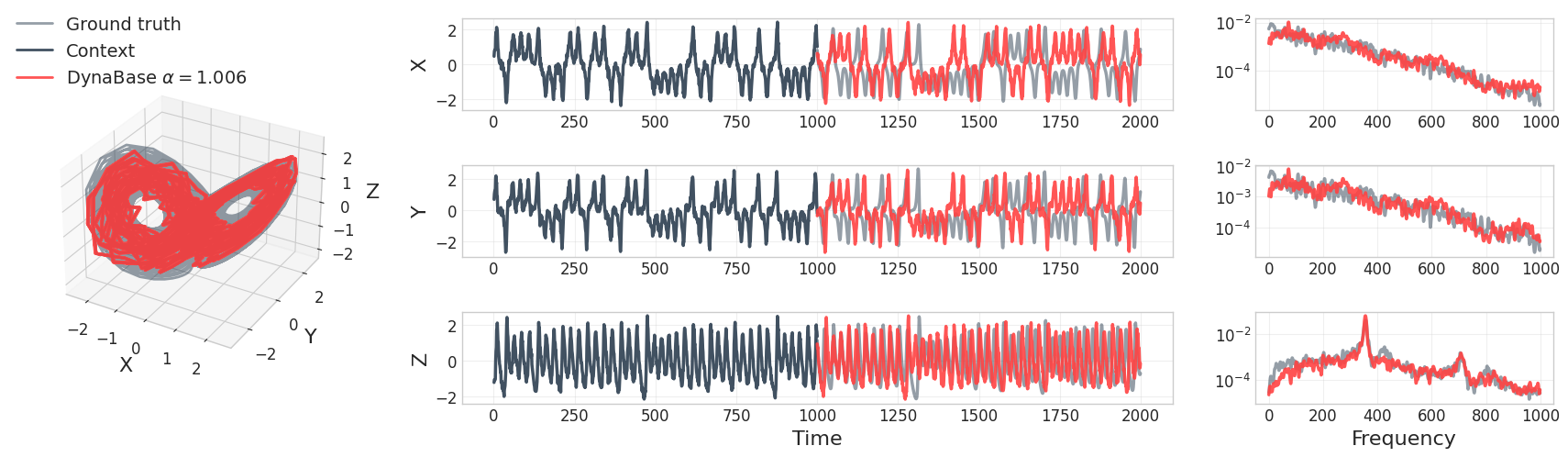}
    \caption{DynaBase reconstruction of Lorenz-63 adding $15\%$ Gaussian noise to the context signal. Left: State space, center: time graphs, right: power spectrum.}
    \label{fig:dynabase_noise}
\end{figure*}

\begin{table}[h]
\centering
\caption{Short-term forecasting accuracy (MASE, median $\pm$ MAD) on the DS test set for three DynaBase training strategies (linear regression (LR), DSR objective, and zero-shot) across various prediction horizons $n$. Best median values in bold. Note that LR training outperforms the other approaches on this measure for intermediate forecasting horizons, while for larger horizons ($n=300$) performance differences narrow down again due to the general limitations caused by exponential trajectory divergence in chaotic systems.}
\vspace{0.5em}
\resizebox{\textwidth}{!}{%
\begin{tabular}{lccccc}
\hline
Model & $\text{MASE}\;(n=10)$ & $\text{MASE}\;(n=50)$ & $\text{MASE}\;(n=100)$ & $\text{MASE}\;(n=200)$ & $\text{MASE}\;(n=300)$ \\
\hline
DynaBase LR & $\bm{0.13\pm0.08}$ & $\bm{0.51\pm0.40}$ & $\bm{0.71\pm0.62}$ & $\bm{1.06\pm0.93}$ & $\bm{2.64\pm2.18}$ \\
DynaBase DSR \& Gating & $\bm{0.13\pm0.08}$ & $0.62\pm0.48$ & $1.06\pm0.89$ & $2.22\pm1.64$ & $2.80\pm1.96$ \\
DynaBase zero-shot & $\bm{0.13\pm0.07}$ & $0.61\pm0.48$ & $0.94\pm0.80$ & $2.29\pm1.69$ & $3.07\pm2.09$ \\\hline
\end{tabular}}\label{tab:dynabse_short-term}
\end{table}

\begin{table}[h]
\centering
\caption{Zero-shot performance of DynaBase across all 54 test DS used in \cite{hemmer2025true}, with the $x$-dimension scaled by a factor of $2$. Reported are median$\pm$MAD for geometrical disagreement $D_\text{stsp}$, temporal disagreement $D_H$, 10-step-ahead prediction $\text{MASE}$, and $\alpha$ estimated via grid search.}
\vspace{0.5em}
\begin{tabular}{lcccc}
\hline
Model & $D_{\text{stsp}}\downarrow$ & $D_H\downarrow$ & $\text{MASE}\downarrow$ & $\alpha$ \\
\hline
DynaBase (original) & $1.76\pm0.62$ & $0.06\pm0.03$ & $0.13\pm0.08$ & $1.009\pm0.006$ \\
DynaBase ($x$-dim.\ scaled $\times2$) & $1.74\pm0.60$ & $0.06\pm0.03$ & $0.14\pm0.10$ & $1.009\pm0.006$ \\
\hline
\end{tabular}\label{tab:scaling}
\end{table}

\begin{table}[h]
\centering
\caption{Zero-shot performance of DynaBase across all 54 test DS used in \cite{hemmer2025true} for different context lengths $T_C$. Reported are median$\pm$MAD for geometrical disagreement $D_\text{stsp}$, temporal disagreement $D_H$, 10-step-ahead prediction $\text{MASE}$, and $\alpha$ estimated via grid search.}
\vspace{0.5em}
\begin{tabular}{lcccc}
\hline
Model & $D_{\text{stsp}}\downarrow$ & $D_H\downarrow$ & $\text{MASE}\downarrow$ & $\alpha$ \\
\hline
DynaBase ($T_C=1000$) & $2.70\pm1.03$ & $0.07\pm0.05$ & $0.22\pm0.16$ & $1.006\pm0.006$ \\
DynaBase ($T_C=2000$) & $1.76\pm0.62$ & $0.06\pm0.03$ & $0.13\pm0.08$ & $1.009\pm0.006$ \\
DynaBase ($T_C=3000$) & $1.31\pm0.50$ & $0.05\pm0.03$ & $0.14\pm0.09$ & $1.009\pm0.006$ \\
\hline
\end{tabular}\label{tab:context_length}
\end{table}

\begin{table}[h]
\centering
\caption{Zero-shot performance of DynaBase across all 54 test DS used in \cite{hemmer2025true} for partially observed DS, using delay-embedding of only the $x$-dimension. Reported are median$\pm$MAD for geometrical disagreement $D_\text{stsp}$, temporal disagreement $D_H$, 10-step-ahead prediction $\text{MASE}$, and $\alpha$ estimated via grid search.}
\vspace{0.5em}
\begin{tabular}{lcccc}
\hline
Model & $D_{\text{stsp}}\downarrow$ & $D_H\downarrow$ & $\text{MASE}\downarrow$ & $\alpha$ \\
\hline
DynaBase (fully observed) & $1.76\pm0.62$ & $0.06\pm0.03$ & $0.13\pm0.08$ & $1.009\pm0.006$ \\
DynaBase (partially observed) & $1.96\pm1.04$ & $0.06\pm0.04$ & $0.16\pm0.12$ & $1.014\pm0.006$ \\
\hline
\end{tabular}\label{tab:partial_obs}
\end{table}

\begin{table}[h]
\centering
\caption{Zero-shot performance of DynaBase across all 54 test DS used in \cite{hemmer2025true} for different levels of Gaussian observation noise on the context signal. Reported are median$\pm$MAD for geometrical disagreement $D_\text{stsp}$, temporal disagreement $D_H$, 10-step-ahead prediction $\text{MASE}$, and $\alpha$ estimated via grid search.}
\vspace{0.5em}
\begin{tabular}{lcccc}
\hline
Model & $D_{\text{stsp}}\downarrow$ & $D_H\downarrow$ & $\text{MASE}\downarrow$ & $\alpha$ \\
\hline
DynaBase (0\% noise) & $1.76\pm0.62$ & $0.06\pm0.03$ & $0.13\pm0.08$ & $1.009\pm0.006$ \\
DynaBase (5\% noise) & $2.17\pm0.67$ & $0.07\pm0.02$ & $0.77\pm0.39$ & $1.020\pm0.001$ \\
DynaBase (10\% noise) & $3.09\pm0.89$ & $0.10\pm0.02$ & $1.34\pm0.55$ & $1.017\pm0.003$ \\
\hline
\end{tabular}\label{tab:noise}
\end{table}

\begingroup
\setlength{\tabcolsep}{2.4pt}
\begin{table*}[!ht]
\centering
\caption{Forecasting results of DynaMix and DynaBase on different real world time series. Reported are median$\pm$MAD (for DynaMix) for geometrical disagreement $D_\text{stsp}$, temporal disagreement $D_H$, 10-step-ahead prediction $\text{MASE}$.}
\begin{tabular}{l | ccc | ccc}
\toprule
\textbf{System} &
\multicolumn{3}{c |}{\textbf{DynaMix}} &
\multicolumn{3}{c}{\textbf{DynaBase}} \\
& $D_{\text{stsp}}$ & $D_H$ & MASE & $D_{\text{stsp}}$ & $D_H$ & MASE \\
\midrule
ETTh1 &
$0.59 \pm 0.002$ & $0.15 \pm 0.003$ & {\boldmath $1.55 \pm 0.02$} &
{\boldmath $0.12$} & {\boldmath $0.13$} & $1.68$ \\
Traffic &
$0.85 \pm 0.02$ & $0.18 \pm 0.005$ & $1.20 \pm 0.44$ &
{\boldmath $0.18$} & {\boldmath $0.12$} & {\boldmath $1.09$} \\
Cloud Requests &
$0.68 \pm 0.01$ & $0.07 \pm 0.006$ & {\boldmath $2.64 \pm 0.12$} &
{\boldmath $0.05$} & {\boldmath $0.05$} & $3.59$ \\
Weather (Temp.) &
$0.06 \pm 0.002$ & {\boldmath $0.07 \pm 0.003$} & {\boldmath $2.04 \pm 0.07$} &
{\boldmath $0.03$} & $0.08$ & $3.17$ \\
Weather (Press.) &
$0.19 \pm 0.004$ & {\boldmath $0.26 \pm 0.008$} & {\boldmath $4.32 \pm 0.08$} &
{\boldmath $0.14$} & $0.27$ & $6.75$ \\
Human fMRI &
$0.24 \pm 0.03$ & $0.20 \pm 0.004$ & {\boldmath $2.03 \pm 0.13$} &
{\boldmath $0.20$} & {\boldmath $0.19$} & $3.02$ \\
Human EEG &
$0.20 \pm 0.01$ & $0.24 \pm 0.02$ & $7.56 \pm 0.25$ &
{\boldmath $0.16$} & {\boldmath $0.21$} & {\boldmath $4.30$} \\
\bottomrule
\end{tabular}\label{tab:TS_eval}
\end{table*}
\endgroup

\clearpage
\subsection{Further theoretical results}
\paragraph{Details on continuous-time limit.}\label{sec:continuous_derivation}
We begin by writing eq. \ref{eq:self-consistent-one-param} explicitly as a piecewise-affine discrete-time system
\begin{equation}
    \bm{z}_{k+1}=f_\alpha(\bm{z}_k)=\alpha\bm{I}\bm{z}_k+(\bm{c}_{\tau(\bm{z}_k)+1}-\alpha\bm{c}_{\tau(\bm{z}_k)}) \ .
\end{equation}
For values of $\alpha\in\mathbb{R}\setminus\{0,1\}$, Thm.~3 in \cite{monfared_transformation_2020} guarantees the existence of an equivalent affine continuous-time system, in the sense that
\begin{equation}
    \bm{\zeta}(t_0)=\bm{z}_{k_0},\;\bm{\zeta}(t_0+h)=\bm{z}_{k_0+1} \ ,
\end{equation}
where $h$ denotes the discrete timestep at which observations $\bm{z}_k$ are sampled. The equivalent continuous-time system is then given by 
\begin{equation}
    \dv{\bm{\zeta}}{t} = \frac{\log\alpha}{h}\bm{\zeta}-\frac{1}{h}\frac{\log\alpha}{1-\alpha}(\bm{c}_{\tau(\bm{\zeta})+1}-\alpha\bm{c}_{\tau(\bm{\zeta})}) \ .
\end{equation}
For the case $\alpha=1$, we define $\frac{\log{\alpha}}{1-\alpha}\rvert_{\alpha=1} :=-1$ by its continuous extension. Such an extension does, however, not exist for $\alpha=0$. For $\alpha<0$, the logarithm is complex-valued. Hence, a real-valued continuous-time limit for $f_\alpha$ exists only for $\alpha>0$.

\paragraph{Computational complexity}
The minimal architecture of DynaBase makes it computationally highly efficient for zero-shot inference. By exploiting discrete nearest-neighbor lookup and efficient heuristic search algorithms such as $k$-d trees \cite{friedman1977algorithm}, the time complexity of running DynaBase autoregressively over $T_F$ forecast steps is
\begin{equation}
    \mathcal{T} = \mathcal{O}(T_F \log T_C),
\end{equation}
with a one-time overhead of $\mathcal{O}(T_C \log T_C)$ for constructing the $k$-d tree. In higher dimensions, query time degrades toward $\mathcal{O}(T_C)$ in the worst case. Provided the generated trajectory is not retained, the space complexity reduces to
\begin{equation}
    \mathcal{S} = \mathcal{O}(T_C).
\end{equation}
These bounds allow DynaBase to generate long forecasts over extended contexts at low computational cost — in contrast to standard transformer architectures, which incur quadratic time and space complexity in context length.

\subsection{Alternative geometric optimization criterion}\label{sec:alternative_optimization}
One may argue that using $D_{\text{stsp}}$ for both training and evaluation yields an unfair advantage for DynaBase compared to other methods on this measure. Although this was a deliberate choice to expose the importance of the training criterion (sect. \ref{sec:results-empirical}), and the same might be said about MASE (on which TS FMs are trained) the other way round, here we study how much performance degrades if DynaBase training is performed on a different geometrical criterion than used for evaluation. Specifically, we used the sliced Wasserstein-1 distance $D_W$ between $p_{\text{true}}(\bm{x})$ and $p_{\text{gen}}(\bm{x}|\bm{z})$ \cite{bonneel2015sliced,durstewitz2026position} for training. It averages one-dimensional Wasserstein-1 distances over $L=200$ random projections $\bm{\theta}_l \in \mathbb{S}^{N-1}$:
\begin{equation}\label{eq:Wasserstein-distance}
    D_W \approx \frac{1}{L}\sum_{l=1}^{L} W_1\bigl(\bm{\theta}_l^{\top}\bm{x}_{\text{true}},\, \bm{\theta}_l^{\top}\bm{x}_{\text{gen}}\bigr).
\end{equation}
As shown in Table \ref{tab:long-term_training}, changing the objective to another DSR criterion leads to similar performance, and still outperforms most other models, including training via linear regression.

\begin{table}[ht!]
\centering
\caption{Reconstruction results on whole 54 DS benchmark dataset for different training objectives (linear regression, KL-divergence $D_{\text{stsp}}$, see eq. \ref{eq:KL_loss}, or Wasserstein-1 distance, see eq. \ref{eq:Wasserstein-distance}) using grid search. Median$\pm$MAD.}
\vspace{0.5em}
\begin{tabular}{lcccc}
\hline
Model & $D_{\text{stsp}}\downarrow$ & $D_H\downarrow$ & $\text{MASE}\downarrow$ & $\alpha$ \\
\hline
DynaBase LR & $4.34\pm2.27$ & $0.06\pm0.05$ & $0.13\pm0.08$ & $0.960 \pm 0.031$ \\
DynaBase DSR ($D_{stsp}$) & $1.76\pm0.62$ & $0.06\pm0.03$ & $0.13\pm0.08$ & $1.009 \pm 0.006$ \\
DynaBase-DSR (Wasserstein) & $2.16\pm1.16$ & $0.06\pm0.04$ & $0.13\pm0.08$ & $1.003 \pm 0.007$ \\
\hline
\end{tabular}\label{tab:long-term_training}
\end{table}

\subsection{Comparison to classical nonlinear forecasting and local approximation methods}\label{sec:classical_models}
We further compare DynaBase to classical, library-based local prediction methods for chaotic time series \cite{farmer_predicting_1987, casdagli1989nonlinear, kantz_nonlinear_2004}. For all of them we provide the context trajectory ($T_C=2000$) as a library of states $\bm{x}_t$ with known successors $\bm{x}_{t+1}$. These methods are classically iterated to produce finite-horizon point forecasts. Here we run the same recursion as a long free trajectory, feeding each prediction back as the next query state.

Specifically, we test the \textit{one-NN successor}, which takes the successor of the current state's nearest neighbor as the next state prediction, while the \textit{k-NN successor} and \textit{distance-weighted k-NN} average the successors of the $k$ nearest neighbors (here $k=3$), either uniformly or weighted by inverse distance \cite{farmer_predicting_1987, kantz_nonlinear_2004}. \textit{Locally linear regression} instead fits an affine map $\bm{x}_{t+1} \approx A\bm{x}_t + \bm{b}$ to the $k=2(N+1)$ nearest neighbors by least squares each single time step and uses that for next-state prediction \cite{farmer_predicting_1987, casdagli1989nonlinear}.

As shown in Table~\ref{tab:classical_methods}, DynaBase achieves the best long-term performance ($D_{\text{stsp}}$, $D_H$), while locally linear regression yields the lowest short-term error. This is expected, as these classical methods were designed for short-term prediction rather than for reconstructing the long-term dynamics of a DS, and all of them lack parameters such as $\alpha$ in DynaBase which can be adjusted to regulate local con- or divergence rates, and they are thus more restricted in the types of global dynamics they can represent given a specific context. We emphasize, however, that DynaBase was not intended to become a new SOTA forecasting and DSR model, but rather was introduced here to distill and understand the minimal mechanisms necessary for zero-shot DSR in DSR FMs like DynaMix.

\begin{table}[h]
\centering
\caption{Comparison of DynaBase to classical nonlinear library-based forecasters across all 54 test DS used in \cite{hemmer2025true}, for the same context length $T_C=2000$ (disregarding diverging runs). Reported are median$\pm$MAD for geometrical disagreement $D_\text{stsp}$, temporal disagreement $D_H$, and 10-step-ahead prediction $\text{MASE}$.}
\vspace{0.5em}
\begin{tabular}{lccc}
\hline
Model & $D_{\text{stsp}}\downarrow$ & $D_H\downarrow$ & $\text{MASE}\downarrow$ \\
\hline
DynaBase & $1.76\pm0.62$ & $0.06\pm0.03$ & $0.13\pm0.08$ \\
One-NN successor & $5.00\pm2.53$ & $0.06\pm0.05$ & $0.19\pm0.13$ \\
k-NN successor & $9.73\pm1.47$ & $0.17\pm0.12$ & $0.24\pm0.18$ \\
Distance-weighted k-NN & $8.51\pm1.99$ & $0.13\pm0.10$ & $0.23\pm0.18$ \\
Locally linear regression & $2.25\pm1.65$ & $0.11\pm0.10$ & $0.01\pm0.01$ \\
\hline
\end{tabular}\label{tab:classical_methods}
\end{table}


\section{Proofs of Theorems}\label{sec:proofs}

In the following we use the same notation as introduced in Sec. \ref{sec:dynabase}, with $\bm{C}=\{\bm{c}_1,\dots,\bm{c}_{T_C}\}\in\mathbb{R}^{N\times T_C}$ a context sequence, $\tau(\bm{z})=\argmin_{s}\norm{\bm{c}_s-\bm{z}}_2$ the nearest-neighbor lookup, and
\begin{equation}
    f_{\alpha,\beta}(\bm{z})=\alpha\bm{z}+\beta\bm{c}_{\tau(\bm{z})}+\gamma\bm{c}_{\tau(\bm{z})+1},\qquad \gamma=1-\alpha-\beta,
\end{equation}
the two-parameter DynaBase map; $f_\alpha$ denotes its self-consistent reduction with $\beta=-\alpha$, $\gamma=1$. Ties in $\tau$ are broken by the smallest index; the tie set has measure zero and does not affect any of the statements below.

\subsection{Proof of Theorem \ref{thm:uniqueness} (Uniqueness of form)}
\begin{proof}
We start from the affine ansatz eq. \ref{eq:linear-general},
\begin{equation}
    f_\text{Linear}(\bm{z},\bm{C})=\bm{A}\bm{z}+\bm{B}^{(1)}\bm{c}_{\tau(\bm{z})}+\bm{B}^{(2)}\bm{c}_{\tau(\bm{z})+1}+\tilde{\bm{h}},
    \label{eq:proof-uniq-ansatz}
\end{equation}
with $\bm{A},\bm{B}^{(1)},\bm{B}^{(2)}\in\mathbb{R}^{N\times N}$ and $\bm{h}\in\mathbb{R}^N$, and impose the two structural assumptions in turn.

\paragraph{Orthogonal equivariance.}
$O(N)$-equivariance requires $Q\,f_\text{Linear}(\bm{z},\bm{C})=f_\text{Linear}(Q\bm{z},Q\bm{C})$ for all $Q\in O(N)$, where $Q\bm{C}:=\{Q\bm{c}_s\}_s$. Since $\tau$ is itself $O(N)$-equivariant, substituting eq. \ref{eq:proof-uniq-ansatz} and matching coefficients of the independent vectors $\bm{z},\bm{c}_{\tau(\bm{z})},\bm{c}_{\tau(\bm{z})+1}$ gives
\begin{equation}
    Q\bm{A}=\bm{A}Q,\quad Q\bm{B}^{(1)}=\bm{B}^{(1)}Q,\quad Q\bm{B}^{(2)}=\bm{B}^{(2)} Q,\quad Q\bm{h}=\bm{h},\qquad \forall Q\in O(N).
\end{equation}
The standard representation of $O(N)$ on $\mathbb{R}^N$ is irreducible for $N\ge 2$, so by Schur's lemma every commuting matrix is a scalar multiple of the identity, and the only $O(N)$-fixed vector is $\bm{0}$:
\begin{equation}
    \bm{A}=\alpha\bm{I},\quad \bm{B}^{(1)}=\beta\bm{I},\quad \bm{B}^{(2)}=\gamma\bm{I},\quad \bm{h}=\bm{0},\qquad \alpha,\beta,\gamma\in\mathbb{R}.
\end{equation}
Hence eq. \ref{eq:proof-uniq-ansatz} reduces to $f(\bm{z})=\alpha\bm{z}+\beta\bm{c}_{\tau(\bm{z})}+\gamma\bm{c}_{\tau(\bm{z})+1}$.

\paragraph{Translation equivariance.}
For any $\bm{v}\in\mathbb{R}^N$, applying the joint shift $\bm{z}\mapsto\bm{z}+\bm{v}$, $\bm{C}\mapsto\bm{C}+\bm{v}$ leaves $\tau$ invariant. Substituting into the previous form,
\begin{equation}
    f(\bm{z}+\bm{v};\bm{C}+\bm{v})=\alpha(\bm{z}+\bm{v})+\beta(\bm{c}_{\tau(\bm{z})}+ \bm{v})+\gamma(\bm{c}_{\tau(\bm{z})+1}+\bm{v})=f(\bm{z};\bm{C})+(\alpha+\beta+\gamma)\bm{v}.
\end{equation}
Translation equivariance demands $f(\bm{z}+\bm{v};\bm{C}+\bm{v})=f(\bm{z};\bm{C})+\bm{v}$ for all $\bm{v}$, hence $\alpha+\beta+\gamma=1$. Setting $\gamma=1-\alpha-\beta$ yields the two-parameter family
\begin{equation}
    f_{\alpha,\beta}(\bm{z})=\alpha\bm{z}+\beta\bm{c}_{\tau(\bm{z})}+(1-\alpha-\beta)\,\bm{c}_{\tau(\bm{z})+1}.
\end{equation}
\end{proof}

\subsection{Proof of Theorem \ref{thm:parroting} (Reduction to context parroting)}
\begin{proof}
For $\alpha=\beta=0$, the DynaBase map on the delay-embedded context $\overline{\bm{C}}_{1:T_C-D+1}\in\mathbb{R}^{D\times (T_C-D+1)}$ reduces to the deterministic lookup
\begin{equation}
    \bm{z}_{t+1}=f_0(\bm{z}_t)=\overline{\bm{c}}_{\tau(\bm{z}_t)+1},\qquad \tau(\bm{z})=\argmin_{1\le s\le T_C-D}\norm{\overline{\bm{c}}_s-\bm{z}}_2. \label{eq:proof-parrot}
\end{equation}

\paragraph{Initial step.}
Starting from $\bm{z}_0=\overline{\bm{c}}_{T_C-D+1}$ (the $D$ most recent context entries), the first application of eq. \ref{eq:proof-parrot}, assuming the final $D$ steps are omitted for $\tau(\bm{z}_0)$ (see \cite{zhang2026context}), yields
\begin{equation}
    s^\star :=\tau(\bm{z}_0)=\argmin_{1\le s\le T_C-2D}\norm{\overline{\bm{c}}_s-\overline{\bm{c}}_{T_C-D+1}}_2.
\end{equation}
Since $\overline{\bm{c}}_s=(c_s,\dots,c_{s+D-1})$, this distance is the Euclidean distance between two $D$-length motifs of the original context---identical to the motif-matching step of \cite[Algorithm~1]{zhang2026context}. Hence $s^\star$ coincides with the best-matching motif index produced by context parroting.

\paragraph{Induction step.}
Suppose $\bm{z}_k=\overline{\bm{c}}_{s^\star+k}$ with $s^\star+k\le T_C-D$. Then $\tau(\bm{z}_k)=s^\star+k$ at zero distance, and
\begin{equation}
    \bm{z}_{k+1}=\overline{\bm{c}}_{s^\star+k+1}=(c_{s^\star+k+1},\dots,c_{s^\star+k+D}).
\end{equation}
By induction, $\bm{z}_t=\overline{\bm{c}}_{s^\star+t}$ for all $0\le t\le T_C-D-s^\star$. Extracting observations from the delay embedded vectors results in
\begin{equation}
    x_t=z_{D,t}=c_{s^\star+t+D-1},
\end{equation}
which is exactly the forecast emitted by \cite[Algorithm~1]{zhang2026context} after locating the best-matching motif $s^\star$.

\paragraph{Wrap-around.}
Once $s^\star+t$ exceeds $T_C-D$, both algorithms treat the current tail as a new query and repeat the motif search. Both procedures execute the same operation, so the equivalence extends to arbitrary forecast horizons.
\end{proof}

\subsection{Proof of Theorem \ref{thm:chaos} (Long-term dynamics)}
\begin{proof} For $\alpha=0$, $f_\alpha$ reduces to $f_0(\bm{z})=\bm{c}_{\tau(\bm{z})+1}$, with finite range $\{\bm{c}_2,\dots,\bm{c}_{T_C}\}$. After one iteration every orbit lies in this finite set, on which $f_0$ acts as the deterministic graph map $\bm{c}_s\mapsto\bm{c}_{\tau(\bm{c}_s)+1}=\bm{c}_{s+1}$ for $s<T_C$. Each node has out-degree exactly one, so iteration on a finite directed functional graph must enter a cycle in at most $T_C-1$ steps and remain periodic thereafter. The differential of $f_0$ vanishes on the interior of every Voronoi cell (the map is piecewise constant), so the Lyapunov exponent \textit{within each cell} is $\lambda_{\text{max}}=-\infty<0$ everywhere. Since the saltation operator $\mathcal{S}_{s,s+1}=\mathbf{I}$ between two Voronoi cells $V_s$ and $V_{s+1}$ is always finite and nonzero for $\alpha=0$, we obtain $\lambda_{\text{max}}=-\infty$ globally. 
\end{proof}

\subsection{Proof of Theorem \ref{thm:boundedness} (Boundedness)}
\begin{proof}
For $\bm{z}\in\mathbb{R}^N$ define the distance to the context points available to the nearest-neighbor lookup by
\begin{equation}
    d_{\bm{C}}(\bm{z}):=\min_{1\le s\le T_C-1}\norm{\bm{c}_s-\bm{z}}_2=\norm{\bm{z}-\bm{c}_{\tau(\bm{z})}}_2.
\end{equation}

\paragraph{Sufficiency ($\alpha\in[0,1)$, and $\alpha=1$ for closed contexts).}
Fix any $\bm{z}_t$ and let $\bm{z}_{t+1}=f_\alpha(\bm{z}_t)=\alpha(\bm{z}_t-\bm{c}_{\tau(\bm{z}_t)})+\bm{c}_{\tau(\bm{z}_t)+1}$ as in eq. \ref{eq:self-consistent-one-param}. If $\tau(\bm{z}_t)+1\le T_C-1$, then $\bm{c}_{\tau(\bm{z}_t)+1}$ is itself a lookup point and 
\begin{equation}
    d_{\bm{C}}(\bm{z}_{t+1})\le \norm{\bm{z}_{t+1}-\bm{c}_{\tau(\bm{z}_t)+1}}_2 = \alpha\,\norm{\bm{z}_t-\bm{c}_{\tau(\bm{z}_t)}}_2 = \alpha\,d_{\bm{C}}(\bm{z}_t).
\end{equation}
If $\tau(\bm{z}_t)+1=T_C$, choose $k\le T_C-1$ with $\norm{\bm{c}_{T_C}-\bm{c}_k}_2=d_{\bm{C}}(\bm{c}_{T_C})$. By the triangle inequality we obtain
\begin{equation}
    d_{\bm{C}}(\bm{z}_{t+1})\le\norm{\bm{z}_{t+1}-\bm{c}_{T_C}}_2+\norm{\bm{c}_{T_C}-\bm{c}_k}_2=\alpha\,d_{\bm{C}}(\bm{z}_t)+d_{\bm{C}}(\bm{c}_{T_C}).
\end{equation}
Hence, in both cases, the following holds:
\begin{equation}
    d_{\bm{C}}(\bm{z}_{t+1})\le \alpha\,d_{\bm{C}}(\bm{z}_t)+d_{\bm{C}}(\bm{c}_{T_C}).
\end{equation}
For closed contexts, $d_{\bm{C}}(\bm{c}_{T_C})=0$ and this gives $d_{\bm{C}}(\bm{z}_{t+1})\le\alpha\,d_{\bm{C}}(\bm{z}_t)\le d_{\bm{C}}(\bm{z}_t)$ for $\alpha\in[0,1]$, so by induction $d_{\bm{C}}(\bm{z}_t)\le d_{\bm{C}}(\bm{z}_0)=:R$ for all $t\ge 0$. For general contexts and $\alpha\in[0,1)$, induction gives $d_{\bm{C}}(\bm{z}_t)\le\alpha^t d_{\bm{C}}(\bm{z}_0)+(1-\alpha^t)\frac{d_{\bm{C}}(\bm{c}_{T_C})}{1-\alpha}\le\max\big\{d_{\bm{C}}(\bm{z}_0),\frac{d_{\bm{C}}(\bm{c}_{T_C})}{1-\alpha}\big\}=:R$ for all $t\ge 0$. In both cases,
\begin{equation}
    \norm{\bm{z}_t}_2\le \max_{1\le s\le T_C-1}\norm{\bm{c}_s}_2+R<\infty,
\end{equation}
and the orbit is bounded. For $\alpha=1$, the closeness condition cannot be dropped: for $\bm{C}=(0,1,2)$, every cell has increment $c_{s+1}-c_s=1$, so $f_1(z)=z+1$ and all orbits diverge. It is, however, not necessary either: for $\bm{C}=(0,1,\varepsilon)$ with $0<\varepsilon\le \tfrac{1}{2}$, $f_1$ maps $[\varepsilon-\tfrac{1}{2},\tfrac{3}{2}]$ onto itself and all other orbits enter this interval after finitely many steps, although $d_{\bm{C}}(\bm{c}_{T_C})=\varepsilon>0$, i.e., boundedness depends on the arrangement of the Voronoi cells.

\paragraph{Non-necessity ($\alpha>1$ counterexample).}
Let $N=1$, $\alpha=2$, and $\bm{C}=(c_1,c_2,c_3)=(-1,1,-1)$. The induced Voronoi cells are $V_1=\{z\le 0\}$ and $V_2=\{z>0\}$, with successors $c_2=1$ and $c_3=-1$, respectively. Substituting into eq. \ref{eq:self-consistent-one-param} gives the piecewise-affine map
\begin{equation}
    f_2(z)=\begin{cases}
        2z+3, & z\le 0,\\
        2z-3, & z>0.
    \end{cases}
\end{equation}
A direct check shows $f_2([-3,3])\subseteq[-3,3]$: for $z\in[-3,0]$ we have $2z+3\in[-3,3]$, and for $z\in(0,3]$ we have $2z-3\in(-3,3]$. Hence any orbit with $z_0\in[-3,3]$ remains in $[-3,3]$ for all $t\ge 0$, proving bounded dynamics for $\alpha=2>1$. This map is topologically conjugate to the doubling/sawtooth map $x_{t+1}=2x_t\bmod 1$, whose Lyapunov exponent equals $\log 2$ on its invariant set, and hence by virtue of this conjugacy inherits the chaotic dynamics of the sawtooth map.
\end{proof}


\end{document}